\PassOptionsToPackage{table}{xcolor} 
\documentclass[runningheads]{llncs}

\usepackage{eccv}

\usepackage{eccvabbrv}

\usepackage{graphicx}
\usepackage{booktabs}

\usepackage[accsupp]{axessibility}  

\usepackage{xr-hyper}
\usepackage{hyperref}

\usepackage{orcidlink}

\usepackage{overpic}
\usepackage{graphicx}
\usepackage{wrapfig}
\usepackage[usestackEOL]{stackengine}

\usepackage[numbers,sort&compress]{natbib}

\makeatletter
\renewcommand\@biblabel[1]{#1.}
\makeatother

\usepackage[capitalize]{cleveref}
\crefname{section}{Sec.}{Secs.}
\Crefname{section}{Section}{Sections}
\crefname{table}{Tab.}{Tabs.}
\Crefname{table}{Table}{Tables}

\makeatletter

\newcommand*{\addFileDependency}[1]{
\typeout{(#1)}
\@addtofilelist{#1}
\IfFileExists{#1}{}{\typeout{No file #1.}}
}\makeatother

\usepackage{tabularx}
\usepackage{etoolbox,siunitx}
\robustify\bfseries

\usepackage{multirow}

\titlerunning{DiffCD: A Symmetric Chamfer Distance for Neural Implicit Surfaces}

\author{Linus Härenstam-Nielsen\orcidlink{0000-0001-6863-4438} \quad
Lu Sang\orcidlink{0009-0007-1158-5584} \quad Abhishek Saroha\\[1mm]
Nikita Araslanov\orcidlink{0000-0002-9424-8837} \quad Daniel Cremers\orcidlink{0000-0002-3079-7984}}

\authorrunning{Härenstam-Nielsen et al.}

\institute{TU Munich \\[1mm]
Munich Center for Machine Learning
}

\newsavebox{\tempbox}

\newcommand \eq[1]{\begin{equation}\begin{aligned}#1\end{aligned}\end{equation}}

\newcommand*\diff{\mathop{}\!\mathrm{d}}

\newcommand{\closest}[2]{x_{\text{proj}}(#1, #2)}
\newcommand{\pull}[2]{x_{\text{pull}}(#1, #2)}

\newcommand{\IGRloss}{L_{\text{IGR}}}
\newcommand{\NPloss}{L_{\text{NP}}}
\newcommand{\DiffCDloss}{L_{\text{DiffCD}}}
\newcommand{\surface}{S}

\newcommand{\SSAloss}{L_{\text{SSA}}}
\newcommand{\sampleSSAloss}{\hat{L}_{\text{SSA}}}
\newcommand{\IGRSSAloss}{L_\text{SIREN}}
\newcommand{\Eikonalloss}{L_{\text{eikonal}}}
\newcommand{\chamfer}{\text{CD}_{\rightarrow}}
\newcommand{\symchamfer}{\text{CD}}
\newcommand{\globalsamples}{{\pointcloud_s}}
\newcommand{\globalsample}{{x_s}}

\newcommand{\mesh}{{M}}
\newcommand{\meshInterval}{{K_\text{mesh}}}
\newcommand{\meshSamples}{{\pointcloud_\text{mesh}}}

\newcommand{\trainIter}{{k}}
\newcommand{\scalar}{{\tau}}

\newcommand\rgt{\aftergroup\mathclose\aftergroup{\aftergroup}\right}
\newcommand{\pointcloud}{\mathcal{P}}

\newcommand{\equal}{\!=\!} 
\newcommand{\der}{\nabla}
\newcommand{\grad}{\nabla_x}

\newcommand*{\inparagraph}[1]{\noindent\textbf{#1}\hspace{0.5em}}

\definecolor{tabfirst}{rgb}{1, 0.7, 0.7} 
\definecolor{tabsecond}{rgb}{1, 0.85, 0.7} 
\definecolor{tabthird}{rgb}{1, 1, 0.7} 

\def\neupull{Neural-Pull\xspace}
\def\siren{SIREN\xspace}
\def\ourshort{DiffCD\xspace}

\newcommand{\argdot}{{{}\cdot{}}}

\newcommand\blfootnote[1]{%
  \begingroup
  \renewcommand\thefootnote{}\footnote{#1}%
  \addtocounter{footnote}{-1}%
  \endgroup
}

\begin{document}
\title{DiffCD: A Symmetric Differentiable Chamfer Distance for Neural Implicit Surface Fitting} 

\maketitle

\begin{abstract}
Neural implicit surfaces can be used to recover accurate 3D geometry from imperfect point clouds.
In this work, we show that state-of-the-art techniques work by minimizing an approximation of a one-sided Chamfer distance.
This shape metric is not symmetric, as it only ensures that the point cloud is near the surface but not vice versa.
As a consequence, existing methods can produce inaccurate reconstructions with spurious surfaces.
Although one approach against spurious surfaces has been widely used in the literature, we theoretically and experimentally show that it is equivalent to regularizing the surface area, resulting in over-smoothing.
As a more appealing alternative, we propose \textit{DiffCD}, a novel loss function corresponding to the \textit{symmetric} Chamfer distance.
In contrast to previous work, DiffCD also assures that the surface is near the point cloud, which eliminates spurious surfaces without the need for additional regularization.
We experimentally show that DiffCD reliably recovers a high degree of shape detail, substantially outperforming existing work across varying surface complexity and noise levels.
\blfootnote{Project code is available at \url{https://github.com/linusnie/diffcd}.}
\end{abstract}

\begin{figure}[ht!]%
    \centering
    
    \begin{subfigure}{0.25\columnwidth}
        \includegraphics[width=\columnwidth]{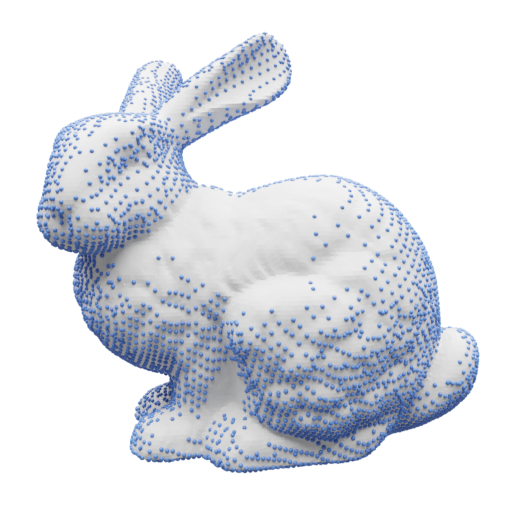}\\[-2mm]
        \includegraphics[width=\columnwidth]{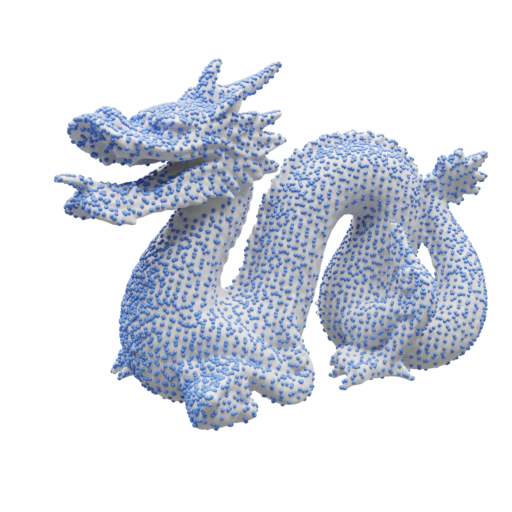}\\[-6mm]
        \includegraphics[width=\columnwidth]{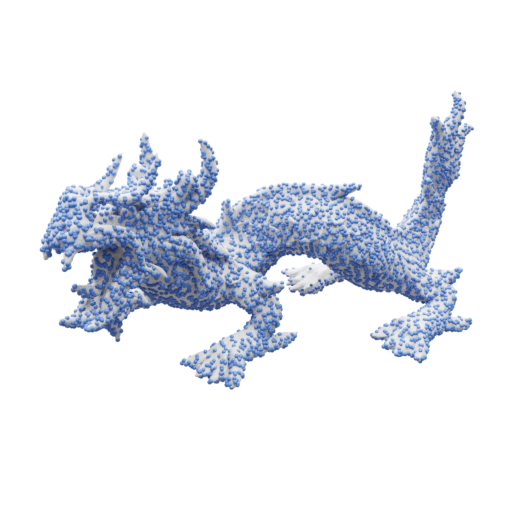}\\[-6mm]
        \includegraphics[width=\columnwidth]{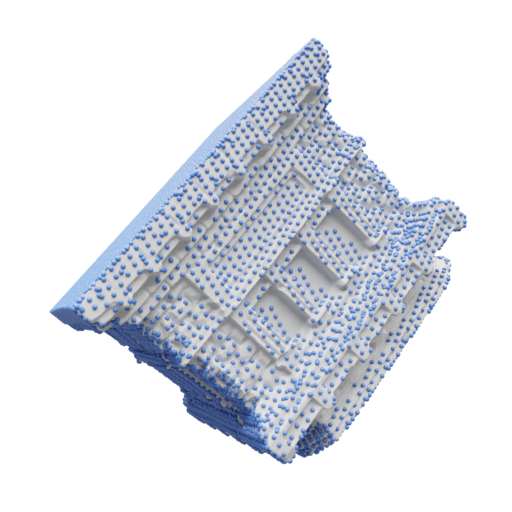}
        \caption{\textbf{Ground truth}}
        \label{subfig:truth}
    \end{subfigure}%
    \begin{subfigure}{0.25\columnwidth}
        \includegraphics[width=\columnwidth]{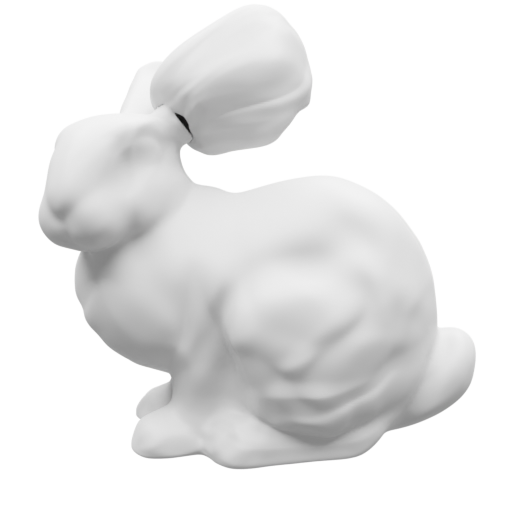}\\[-2mm]
        \includegraphics[width=\columnwidth]{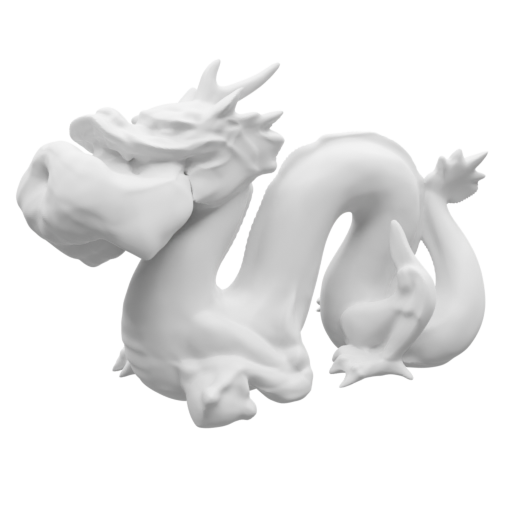}\\[-6mm]
        \includegraphics[width=\columnwidth]{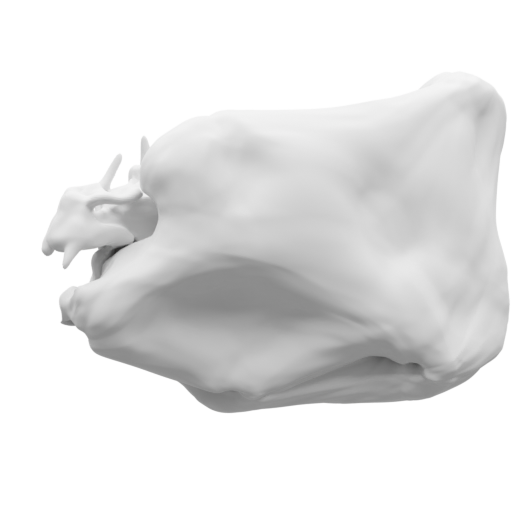}\\[-6mm]
        \includegraphics[width=\columnwidth]{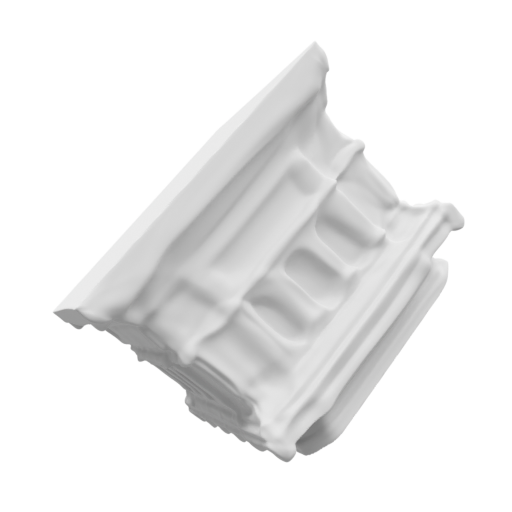}
        \caption{$\mathbf{L}_\textbf{IGR}$ \cite{gropp2020implicit}}
        \label{subfig:igr_teaser}
    \end{subfigure}%
    \begin{subfigure}{0.25\columnwidth}
        \includegraphics[width=\columnwidth]{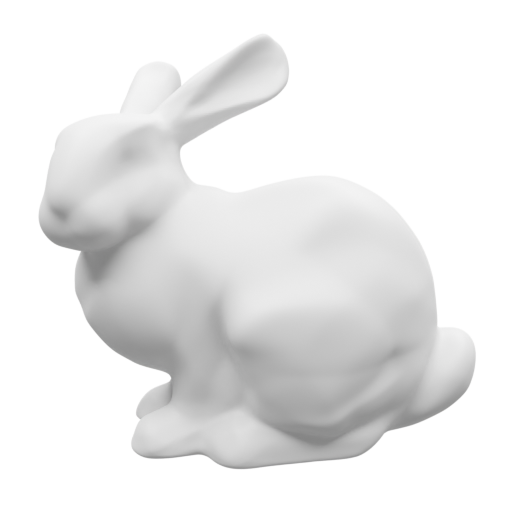}\\[-2mm]
        \includegraphics[width=\columnwidth]{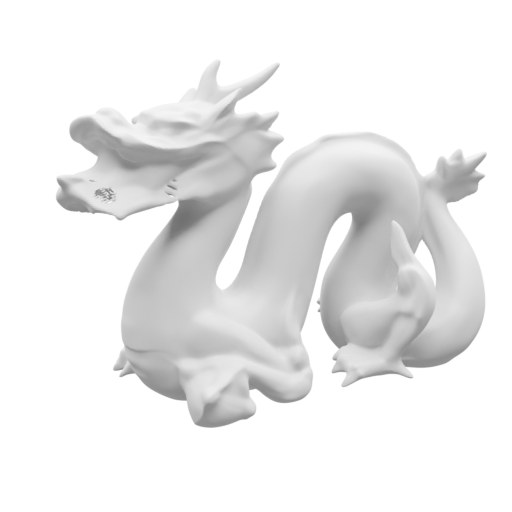}\\[-6mm]
        \includegraphics[width=\columnwidth]{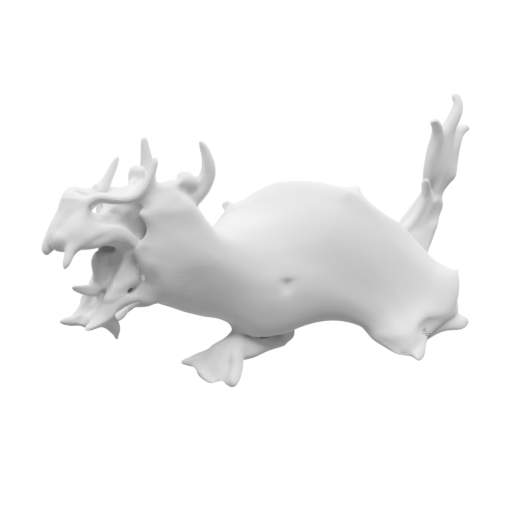}\\[-6mm]
        \includegraphics[width=\columnwidth]{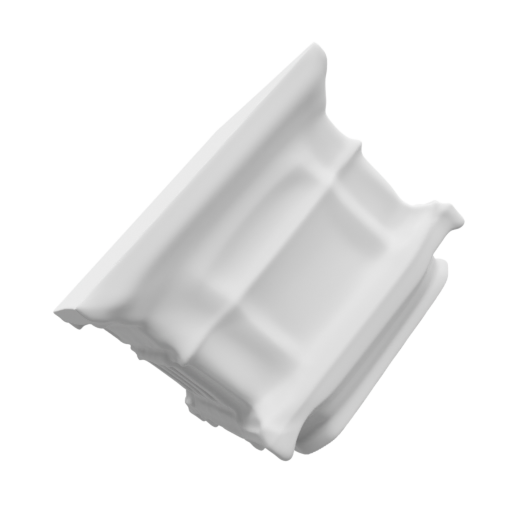}
        \caption{$\mathbf{L}_\textbf{SIREN}$ \cite{sitzmann2019siren}}
        \label{subfig:siren_teaser}
    \end{subfigure}%
    \begin{subfigure}{0.25\columnwidth}
        \includegraphics[width=\columnwidth]{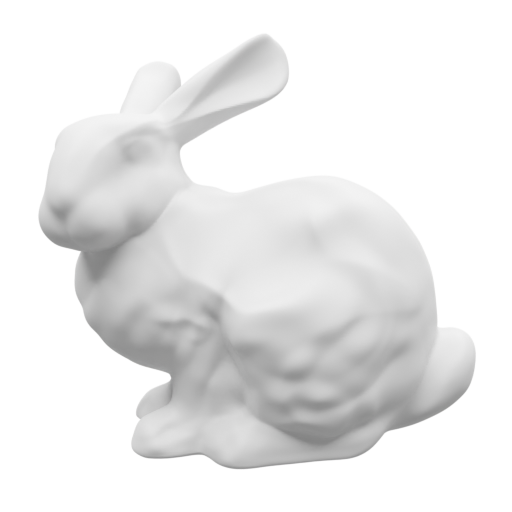}\\[-2mm]
        \includegraphics[width=\columnwidth]{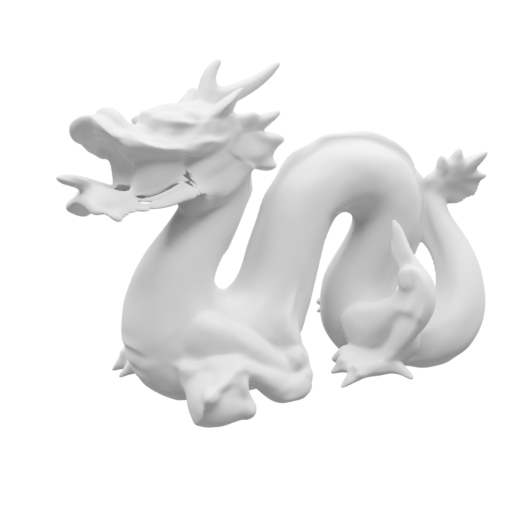}\\[-6mm]
        \includegraphics[width=\columnwidth]{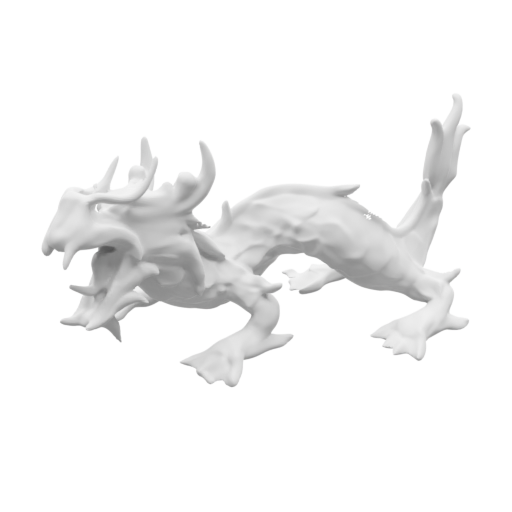}\\[-6mm]
        \includegraphics[width=\columnwidth]{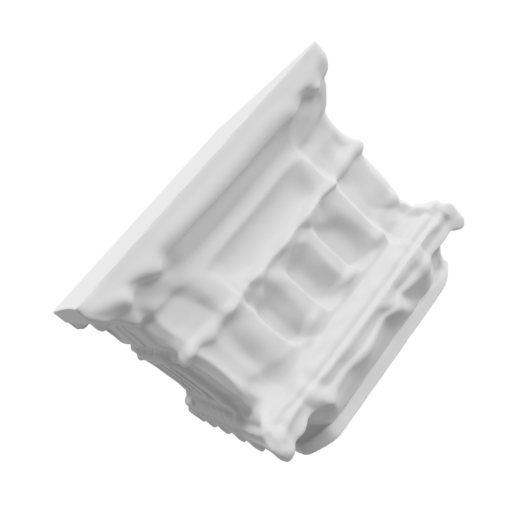}
        \caption{$\mathbf{L}_\textbf{DiffCD}$ \textbf{(ours)}}
        \label{subfig:ours_teaser}
    \end{subfigure}

    \caption{\textbf{Surface reconstruction.} Comparison between optimization-based methods for neural implicit surface reconstruction. 
    \subref{subfig:truth}) Input point cloud and ground-truth surface. 
    \subref{subfig:igr_teaser}) The IGR \cite{gropp2020implicit} loss function may produce spurious surfaces. \subref{subfig:siren_teaser}) The SIREN \cite{sitzmann2019siren} loss function (here with $\mu=0.33$) trades off the spurious artifacts against over- or under-smoothing of the surface. \subref{subfig:ours_teaser}) Our symmetric differentiable Chamfer distance resolves the issue of spurious surfaces while recovering high-quality surface details.}
    \label{fig:spurious-igr}
\end{figure}


\section{Introduction}

Neural implicit surfaces parametrize a 3D surface as the zero-level set of a scalar field predicted by a multi-layer perceptron (MLP) $f(\theta, \argdot):\mathbb{R}^3\rightarrow\mathbb{R}$:
\eq{
\label{eq:zero_level}
    \surface_\theta = \{x : f(\theta, x) = 0\},
}
where $\theta$ are the network parameters.
Neural implicit surfaces are naturally suited for surface reconstruction.
In contrast to meshes, they are continuous everywhere and can represent shapes of arbitrary topology;
in contrast to point clouds, they provide surface normal information and enable occlusion reasoning.
The implicit surface representation is also malleable and allows for complex shape manipulations, such as topological changes \cite{Yang2021GeometryProcessing}. 
Yet, as \cref{fig:spurious-igr} illustrates, optimizing over such a flexible parametrization of the surface also comes with some difficulties.
State-of-the-art techniques \cite{Ma2020NeuralPull,gropp2020implicit,sitzmann2019siren} struggle to accurately recover surfaces in practical scenarios, especially when the sensor data is incomplete and noisy.

\inparagraph{Problem setup.} We consider the problem of reconstructing an underlying surface $\surface$ from an unoriented, sparse and possibly noisy point cloud $\pointcloud = \{\tilde x_i\}_{i = 1}^n$, contained in a bounding region $\Omega$. 
Our goal is to estimate the network parameters $\theta$ such that the zero-level set $\surface_\theta$ approximates the true surface, \ie $\surface_\theta\approx\surface$.
We further require that $f(\theta, \argdot)$ approximates a \textit{Signed Distance Field} (SDF), characterized by the \emph{eikonal equation}:
\eq{
\label{eq:eikonal_eq}
    \|\grad f(\theta, x)\| = 1.
}
This ensures that the absolute value of the field at any point equals the distance to the surface.
Training $f(\theta, \argdot)$ to satisfy the eikonal equation adds regularity to the optimization problem by preventing the degenerate solution $f(\theta, \argdot) = 0$.

\inparagraph{Contributions.}
Firstly, we show that the widely used approach of training the network output to vanish on the point cloud \cite{gropp2020implicit,sitzmann2019siren, BenShabat2021DiGS, wang2023aligningHessianGradient, zixiong23neuralsingular} corresponds to minimizing only one side of the Chamfer distance --
the average distance from the input point cloud to the implicit surface.
We will refer to this quantity as the \textit{points-to-surface} Chamfer distance.
Although minimizing the one-sided Chamfer distance works in some scenarios, it may lead to \emph{spurious surfaces}, see \cref{subfig:igr_teaser}.

Spurious surfaces are large surface chunks located in the ``blind spot'' of the one-sided Chamfer distance.
This observation suggests that minimizing the full, \textit{symmetric} Chamfer distance could lead to a more accurate surface reconstruction.
As a second contribution, we therefore develop a novel loss term corresponding to the ``missing'' one-sided \textit{surface-to-points} Chamfer distance.
This term measures the average distance from the implicit surface to the point cloud.
Our proposed loss function combines the two one-sided Chamfer distances into a single loss function, the \textbf{Diff}erentiable \textbf{C}hamfer \textbf{D}istance (DiffCD).
By minimizing the symmetric Chamfer distance, DiffCD reliably mitigates the issue of spurious surfaces, achieving a substantial improvement in surface reconstruction accuracy across various types of shapes and noise levels.

As our third contribution, we theoretically justify DiffCD against a popular alternative for mitigating spurious surfaces, the off-surface loss adopted by SIREN \cite{sitzmann2019siren}.
We prove that minimizing the SIREN loss term is equivalent to regularizing the surface area of the implicit surface. This implies that the SIREN approach mitigates spurious surfaces at the cost of shrinking the surface area, resulting in over-smoothing, or even to the surface disappearing.

Overall, our work provides novel theoretical insights into surface reconstruction with neural implicit surfaces.
Building on these insights, we develop DiffCD.
It minimizes the \emph{symmetric} Chamfer distance between the point cloud and the implicit surface, leading to a more reliable surface reconstruction.

\section{Related work}
\label{sec:related}

The two dominant categories of techniques for reconstructing neural implicit surfaces from point clouds are \textit{supervised} and \textit{optimization-based}.
Here, we provide a condensed summary of the most relevant previous work.
For a broader overview, we refer interested readers to the recent survey by \citet{sulzer2023survey}.

\inparagraph{Supervised methods.} 
Supervised methods approach the surface reconstruction problem by learning from a large collection of point clouds with corresponding ground-truth meshes.
Supervised methods typically extract feature vectors from the point cloud, which are used to predict either occupancy or a signed distance value.
Points2Surf \cite{Erler2020Points2Surf} employs PointNet \cite{qi2017pointnet} and regresses the sign and distance separately by combining a fine-grained local descriptor with a coarse-grained global descriptor.
POCO \cite{Boulch:2022:POCO} instead uses point cloud convolutions \cite{boulch2020fka}.
Neural Kernel Surface Reconstruction (NKSR) \cite{huang2023nksr} defines the surface as the zero-level set of a Neural Kernel Field \cite{williams2022neuralKernelField}.
It relies on PointNet-like \cite{qi2017pointnet} encoder to extract a voxel hierarchy of kernel features from the point cloud.

\inparagraph{Optimization-based methods.}
More relevant to our work are optimization-based methods \cite{park2019deepsdf, Atzmon:2020:SAL, sitzmann2019siren, atzmon2021sald, Ma2020NeuralPull, lipman2021phase}.
They reconstruct the surface by formulating a loss function, which is optimized from scratch at inference time.
IGR \cite{gropp2020implicit}, Neural-Pull \cite{Ma2020NeuralPull} and SIREN \cite{sitzmann2019siren} all employ different loss functions to ensure that the zero-level goes through the point cloud points.
Leveraging normals in the input significantly simplifies the problem, leading to more accurate reconstruction with classical \cite{Kazhdan:2013:SPS} and neural-based \cite{gropp2020implicit,sitzmann2019siren,atzmon2021sald,Peng:2021:SAP} formulations alike.
Surface reconstruction without normals -- the scenario addressed in this work -- is highly ill-posed.
It inspires interest in regularization techniques, which typically encourage smooth solutions by means of additional loss terms.
Several methods penalize deviations from the SDF properties, most notably the eikonal equation \cite{gropp2020implicit, sitzmann2019siren}.
A related property, which can be derived from the eikonal equation, is that the gradient of an SDF lies in the nullspace of the Hessian~\cite{Mayost2014SurfaceEquations}.
This property has been used to motivate regularizing the determinant of the Hessian \cite{zixiong23neuralsingular} and the norm of the Hessian-gradient product \cite{wang2023aligningHessianGradient}.
Level-set alignment \cite{ma2023towards} also  makes use of this property indirectly by regularizing the misalignment of the gradient in the query points and their level-set projections. DiGs \cite{BenShabat2021DiGS} goes beyond SDF properties and minimizes the divergence of the scalar field, which increases surface smoothness by encouraging flat level sets.
PHASE \cite{lipman2021phase} and, as we will see in \cref{sec:ssa-surface-area-regularizer}, SIREN \cite{sitzmann2019siren} both include loss functions that regularize the surface area of the implicit surface.
In contrast to regularization, our loss term is geometric in nature. In fact, we find that tuning the weight of the eikonal loss is sufficient for adapting our method to various noise levels.

\inparagraph{Chamfer distance for shape reconstruction.}
Point2Mesh \cite{Hanocka2020Point2Mesh} and SAP \cite{Peng:2021:SAP} also develop differentiable variants of the symmetric Chamfer distance, but use an intermediate mesh representation.
In contrast, our DiffCD is defined directly in terms of the implicit surface.

\section{Background}
\subsection{Chamfer distance}
As the base of our analysis and our novel loss formulation, we leverage a widely used metric in shape reconstruction, the Chamfer distance\footnote{The Chamfer distance is not a distance \emph{metric} in the mathematical sense, as it does not satisfy the triangle inequality; it is an average over point-to-set distances.} \cite{fan2017pointChamfer3D}.
The Chamfer distance between two point sets, $\pointcloud_1$ and $\pointcloud_2$, is given by:
\eq{
\label{eqn:sym-p2p}
    \symchamfer(\pointcloud_1, \pointcloud_2) = \frac{1}{2}\Big(\chamfer(\pointcloud_1, \pointcloud_2) + \chamfer(\pointcloud_2, \pointcloud_1)\Big).
}
The term $\chamfer(\pointcloud_1, \pointcloud_2)$ is the \textit{one-sided} Chamfer distance from $\pointcloud_1$ to $\pointcloud_2$.
It simply measures the average distance between each point in $\pointcloud_1$ and its corresponding closest point in $\pointcloud_2$:
\eq{
\label{eqn:p2p}
    \chamfer(\pointcloud_1, \pointcloud_2) = \frac{1}{|\pointcloud_1|}\sum_{x_1 \in \pointcloud_1} \underset{x_2 \in \pointcloud_2}{\text{min}} \| x_1 - x_2 \|,
}
where $\|\argdot\|$ is the $\ell^2$-norm and $|\pointcloud_1|$ is the number of points in $\pointcloud_1$.
Although the one-sided Chamfer distance can be a reliable similarity metric if $\pointcloud_1$ and $\pointcloud_2$ have similar density and size, it can be misleadingly small even in cases of substantial shape mismatch.
Notably, points in $\pointcloud_2$ which are not the nearest neighbor of any point in $\pointcloud_1$ will have no impact on the value of $\chamfer(\pointcloud_1, \pointcloud_2)$.

\subsection{IGR and spurious surfaces}
\label{sec:igt-chamfer}
Let us recall implicit geometric regularization (IGR) \cite{gropp2020implicit}, which serves as the base loss function for several follow-up works \cite{sitzmann2019siren, BenShabat2021DiGS, zixiong23neuralsingular, wang2023aligningHessianGradient}.
In the normal-free unoriented setting, IGR tackles the surface reconstruction problem by minimizing the following loss function:
\eq{
\label{eqn:igr-loss}
    \IGRloss(\theta) \!=\!
        \frac{1}{n}\sum_{i = 1}^{n}|f(\theta, \tilde x_i)| + \lambda\Eikonalloss(\theta).
}
The first term in \cref{eqn:igr-loss} ensures that the network output is close to $0$ on the points in the point cloud, $\tilde x_i \in \mathcal{P}$.
The second term is the \textit{eikonal loss}.
It ensures that $f(\theta, \argdot)$ approximates an SDF by penalizing deviations from the eikonal equation in \cref{eq:eikonal_eq}:
\eq{
\label{eqn:eikonal-loss}
    \Eikonalloss(\theta) = \frac{1}{|\globalsamples|}\sum_{\globalsample\in\globalsamples}(\|\grad f(\theta, \globalsample)\| \!-\! 1)^2.
}
$\globalsamples$ is the set of sample points at which the eikonal loss should be computed.
The hyperparameter $\lambda$ in \cref{eqn:igr-loss} is the \textit{eikonal weight} and it determines the trade-off between the two terms.
As we discuss in \cref{sec:eikonal_weight}, the value of $\lambda$ can have a significant impact on the reconstructed surface.

\inparagraph{IGR approximates a one-sided Chamfer distance.}
Using the fact that $f(\theta, \argdot)$ will approximate the signed distance field of $\surface_\theta$, we can adopt a geometric interpretation of \cref{eqn:igr-loss}.
Namely, when the eikonal loss is small, each term in the sum will approximate a distance to the implicit surface:
\eq{
\label{eqn:approx-sdf-eikonal}
    |f(\theta, \tilde x)|\approx \min_{x\in\surface_\theta}\|\tilde x - x\|
    \text{, when } \Eikonalloss(\theta)\approx 0.
}
Plugging this approximation back into \cref{eqn:igr-loss}, it becomes clear that $\IGRloss$ in turn approximates the average distance between the point cloud and the implicit surface, which we can identify as a one-sided Chamfer distance (\cf \cref{sec:igt-chamfer}).
That is, we have:
\eq{
\label{eqn:IGRloss-exact-sdf}
    \IGRloss(\theta) \approx 
    \frac{1}{n}\sum_{i = 1}^{n}\min_{x\in\surface_\theta}\|\tilde x_i - x\|^p =
    \chamfer(\pointcloud, \surface_\theta)
    \text{, when } \Eikonalloss(\theta)\approx 0.
}
Note that \cref{eqn:IGRloss-exact-sdf} expands the definition of the Chamfer distance from \cref{eqn:p2p}: the minimum here is taken over a continuous surface $S_\theta$.
We will refer to \cref{eqn:IGRloss-exact-sdf} as the \textit{points-to-surface} Chamfer distance between the point cloud and the implicit surface.

\inparagraph{Spurious surfaces.} A critical shortcoming of minimizing only one side of the Chamfer distance is that it cannot account for the surface parts far away from every point in $\pointcloud$.
As illustrated in \cref{fig:spurious-igr} and \cref{fig:2d-diffcd}, this leads to \textit{spurious surfaces} in the reconstructions produced by IGR.
Specifically, large blob-shaped artifacts tend to accumulate near repeated thin structures and in the areas where the point cloud has missing coverage.
Introduced next, our approach addresses this issue.

\section{DiffCD: A Symmetric Differentiable Chamfer Distance}
\label{sec:diffcd}
\begin{figure}[t]
    \newcommand{\figwidth}{.17\columnwidth}
    \centering
    \begin{subfigure}[t]{0.03\columnwidth}
        \rotatebox[origin=l]{90}{\makebox[0.5cm]{}}%
    \end{subfigure}%
    \begin{subfigure}{.23\columnwidth}
        \begin{overpic}[width=\linewidth, trim={0cm .0cm 0cm .0cm}, clip]{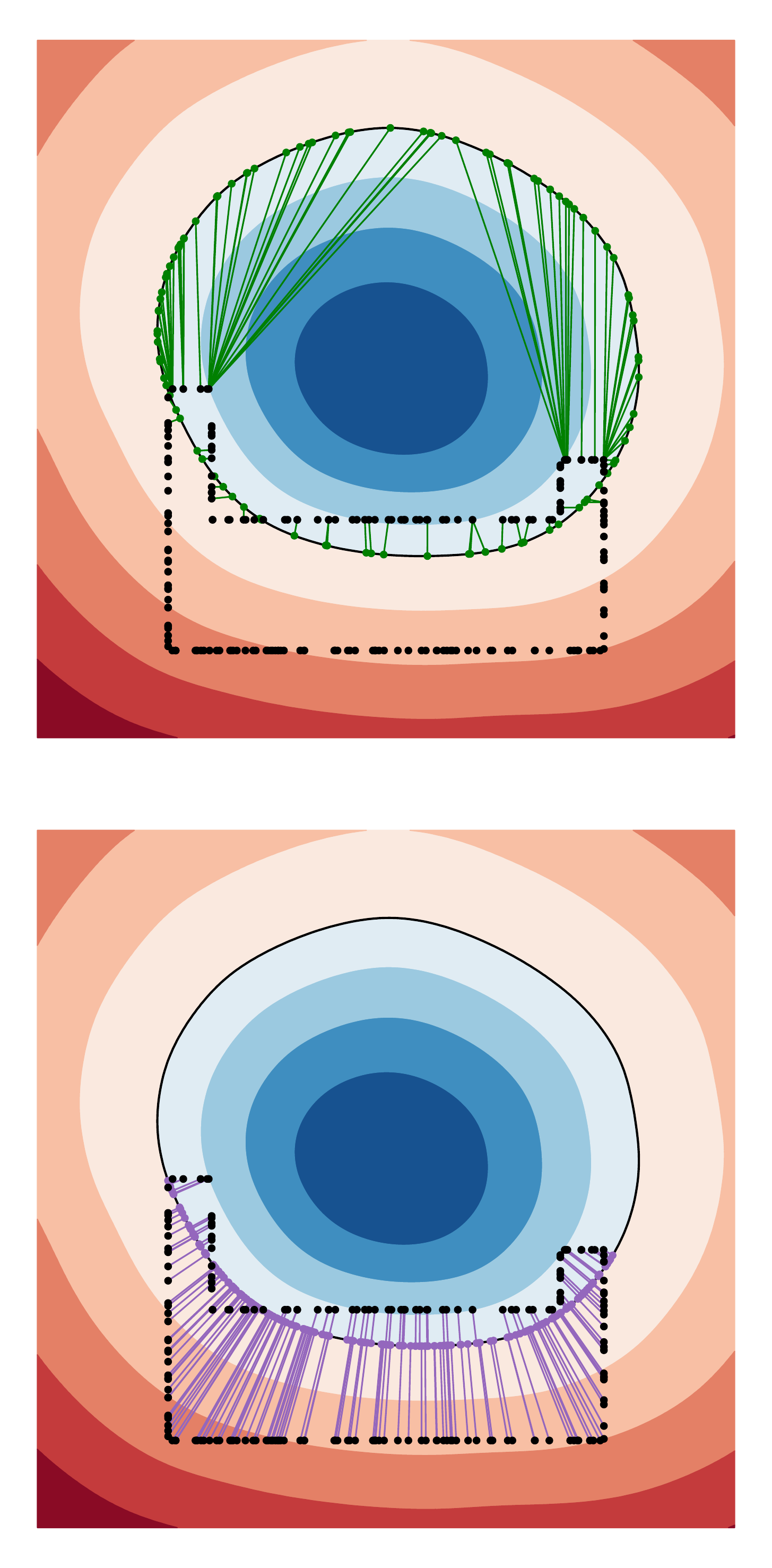}
            \put(-2, 75){\makebox(0,0){\rotatebox{90}{\Centerstack{
                Surface-to-Points
            }}}}
            \put(-3, 25){\makebox(0,0){\rotatebox{90}{\Centerstack{
                Points-to-Surface
            }}}}
        \end{overpic}
        \caption{\textbf{Initialization}}
        \label{subfig:2d-initial}
    \end{subfigure}
    \begin{subfigure}{.23\columnwidth}
        \includegraphics[width=\linewidth]{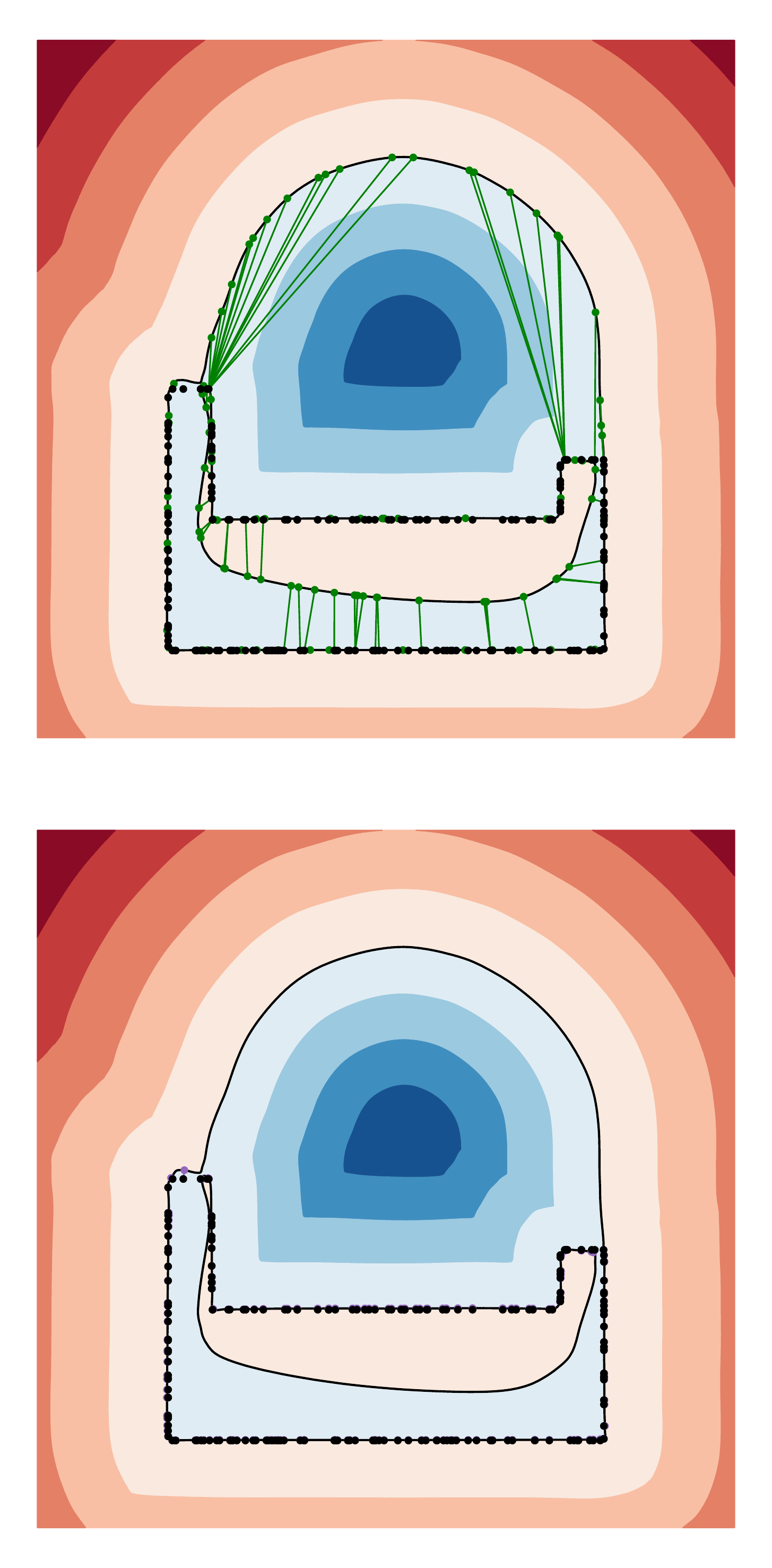}
        \caption{$\mathbf{L}_\textbf{IGR}$ \cite{gropp2020implicit}}
        \label{subfig:2d-igr}
    \end{subfigure}
    \begin{subfigure}{.23\columnwidth}
        \includegraphics[width=\linewidth]{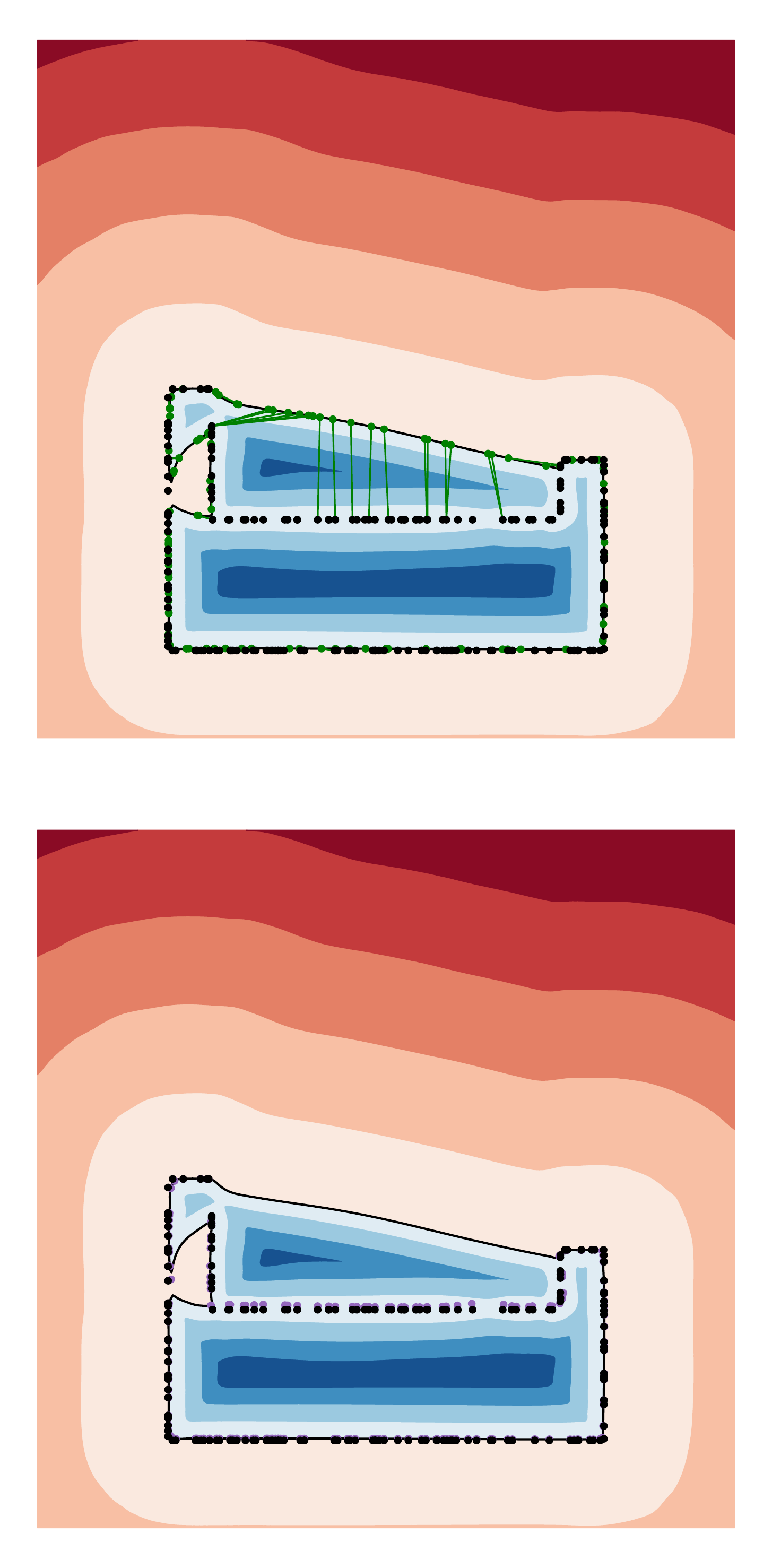}
        \caption{$\mathbf{L}_\textbf{SIREN}$ \cite{sitzmann2019siren}}
        \label{subfig:2d-siren}
    \end{subfigure}
    \begin{subfigure}{.23\columnwidth}
        \includegraphics[width=\linewidth]{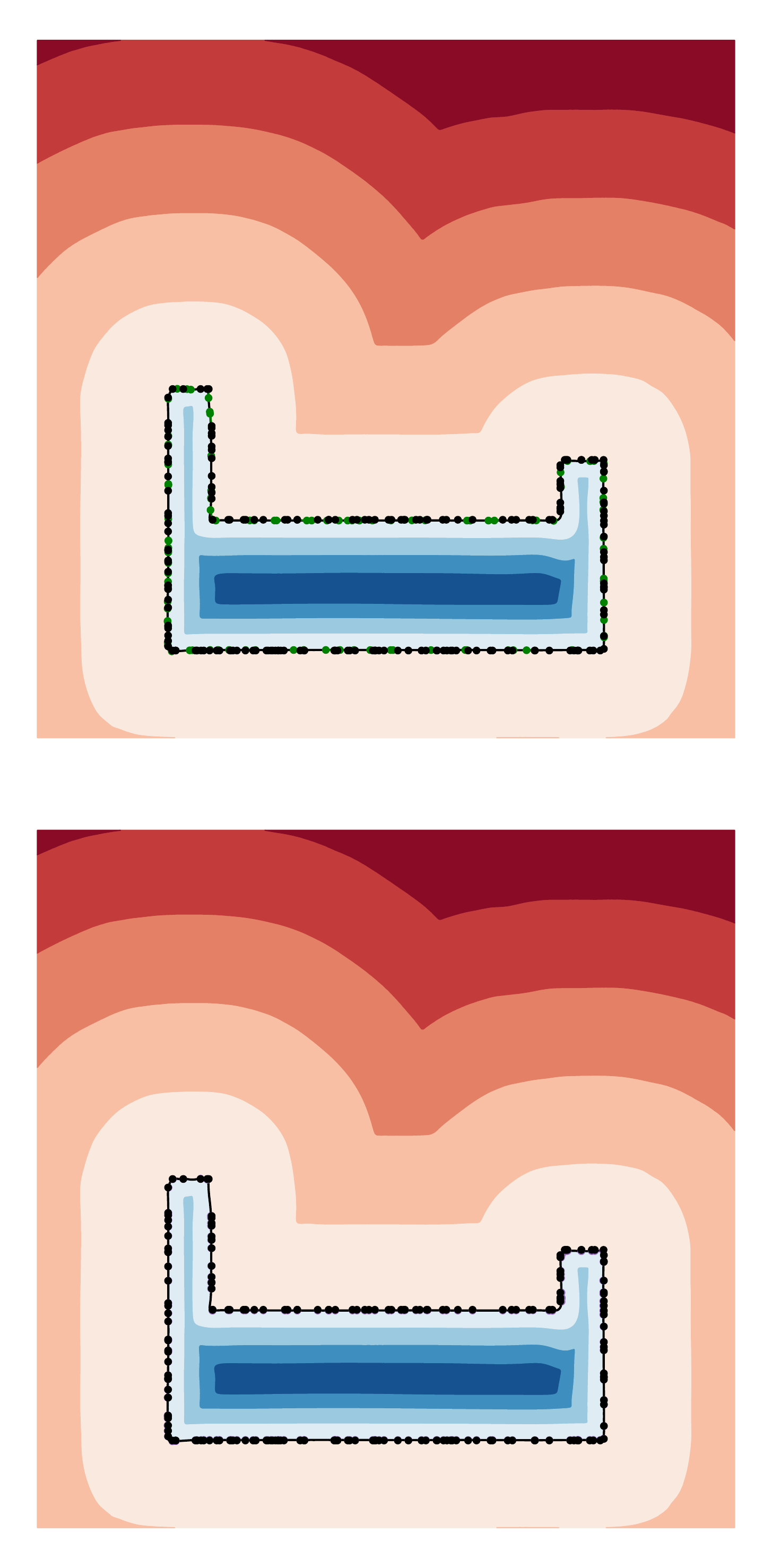}
        \caption{$\mathbf{L}_\textbf{DiffCD}$ \textbf{(ours)}}
        \label{subfig:2d-diffcd}
    \end{subfigure}
    \caption{\textbf{2D surface reconstruction.} Fitting a neural implicit surface to 
    a 2D point cloud. 
    \subref{subfig:2d-initial}) Initial surface.
    \subref{subfig:2d-igr}) IGR \cite{gropp2020implicit} only minimizes the points-to-surface chamfer distance, leaving a large spurious surface even after many optimizer steps.
    \subref{subfig:2d-siren}) Adding the loss term proposed in SIREN \cite{sitzmann2019siren} removes the spurious surface, but also produces an overly smooth solution.
    \subref{subfig:2d-diffcd}) Our method, DiffCD, also minimizes the missing surface-to-points distance, resulting in a tight fit.
    }
    \label{fig:2d-diffcd}
\end{figure}

In order to resolve the issue of spurious surfaces, we propose extending the one-sided Chamfer distance in \cref{eqn:igr-loss} to a full \emph{symmetric} Chamfer distance.
By doing so, we explicitly account for the spurious surfaces with the loss function.
To this end, we add the ``missing'' Chamfer term, namely the \textit{surface-to-points} Chamfer distance.
The surface-to-points Chamfer distance is the average distance from the surface to the point cloud:
\eq{
\label{eqn:s2p-loss}
    \chamfer(\surface_\theta, \pointcloud) = \frac{1}{n}\sum_{i=1}^n\min_{j=1,\ldots,N}\|x_i(\theta) - \tilde x_j\|,
}
where $x_i(\theta)$ are points sampled uniformly from the implicit surface $S_\theta$.
Note that \cref{eqn:s2p-loss} is an \textit{explicit} loss; it is defined geometrically and does not assume $f(\theta, \argdot)$ to be an SDF.
Below, we propose an efficient sampling procedure and a method for computing the gradients of the surface-to-points loss.

\cref{fig:2d-diffcd} illustrates a 2D example of jointly minimizing both sides of the Chamfer distance.
The sample points $x_i(\theta)$ depend on the network parameters $\theta$ and are constrained to lie on the implicit surface.
By minimizing \cref{eqn:s2p-loss}, the gradient descent will tend to move each surface point $x_i(\theta)$ toward the closest point in the point cloud.
This surface pulling toward the point cloud will eventually remove any spurious surfaces by molding them into the existing surface.

Having defined the surface-to-points loss, we introduce \emph{Differentiable Chamfer Distance}, or \emph{DiffCD}.
Similar to the Chamfer distance for point clouds, we take the average of the two one-sided Chamfer distances to obtain:
\eq{
\label{eqn:diffcd-loss}
    \DiffCDloss(\theta) &= \frac{1}{2}\bigg(\frac{1}{n}\sum_{i=1}^n|f(\theta, \tilde x_i)| 
    +
    \chamfer(\surface_\theta, \pointcloud)\bigg)
    +\lambda\Eikonalloss(\theta).
}
Using \cref{eqn:IGRloss-exact-sdf}, we can confirm that \cref{eqn:diffcd-loss} indeed approximates the symmetric Chamfer distance when the eikonal loss is small:
\eq{
    \DiffCDloss(\theta)\approx\symchamfer(\surface_\theta, \pointcloud)\text{, when } \Eikonalloss(\theta)\approx 0.
}
That is, minimizing our proposed loss function \cref{eqn:diffcd-loss} implicitly minimizes the fully symmetric Chamfer distance between the point cloud and the surface, as opposed to just the one-sided Chamfer distance.

\subsection{Computing gradients}
In order to minimize \cref{eqn:s2p-loss} with backpropagation, we need a method for computing surface point derivatives, $\der_\theta x_i(\theta)$, where $\theta\in\mathbb{R}$ is a scalar network parameter.
For this, we use the method for controlling level sets proposed by \citet{atzmon2019controlling}.
The key idea is to make use of the level set equation $f(\theta, x_i(\theta)) = 0$, which must hold for all $\theta$. Differentiating both sides of the level set equation \wrt $\theta$ and rearranging terms yields the constraint:
\eq{
    g^T\der_\theta x_i(\theta) = -f_\theta,
}
where we have used the shorthands $g := \grad f(\theta, x_i(\theta))$ and $f_\theta := \der_\theta f(\theta, x_i(\theta))$.
That is, the gradient $\der_\theta x_i(\theta)$ lies on a plane with normal $g$ at a distance $|f_\theta| / \|g\|$ from the origin. So we can set $\der_\theta x_i(\theta) = \xi-f_\theta g/\|g\|^2$, where $\xi\in\mathbb{R}^3$ is any vector orthogonal to $g$.

Similarly to \citet{atzmon2019controlling}, we take the minimum norm solution, $\xi = 0$.
While this choice may seem arbitrary, we find that it has a significant advantage when combined with gradient descent. 
Namely, it will naturally lead to linearly interpolating surfaces when the input point cloud is sparse.
To see this, we note that the derivative of each term in \cref{eqn:s2p-loss} is:
\eq{
    \label{eqn:sample-distance-grad}
    \der_\theta\|x_i(\theta) - \tilde x\| = -f_\theta\frac{(x_i(\theta) - \tilde x)^Tg}{\|x_i(\theta) - \tilde x\|\|g\|^2}.
}
Importantly, \cref{eqn:sample-distance-grad} has the property that $\der_\theta\|x_i(\theta) - \tilde x\| = 0$ whenever the residual vector $x_i(\theta) - \tilde x$ is orthogonal to the surface normal $g/\|g\|$.
This means that flat surfaces are stable solutions, in the sense that the gradient of \cref{eqn:s2p-loss} will be 0 if the closest point to each sample point lies on the plane orthogonal to the surface normal at the sample point.

\subsection{Sampling $S_\theta$}
\begin{wrapfigure}{r}{0.42\textwidth}
        \centering
        \includegraphics[width=.49\linewidth]{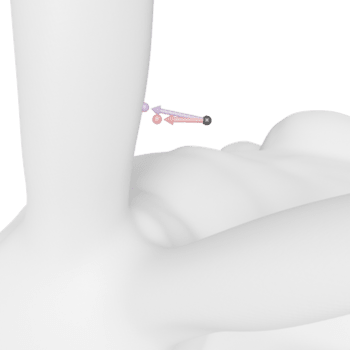}
        \rotatebox{0}{\begin{overpic}[
            width=.49\linewidth,
            angle=270,
            trim={.3cm .3cm .3cm .3cm}, clip,
        ]{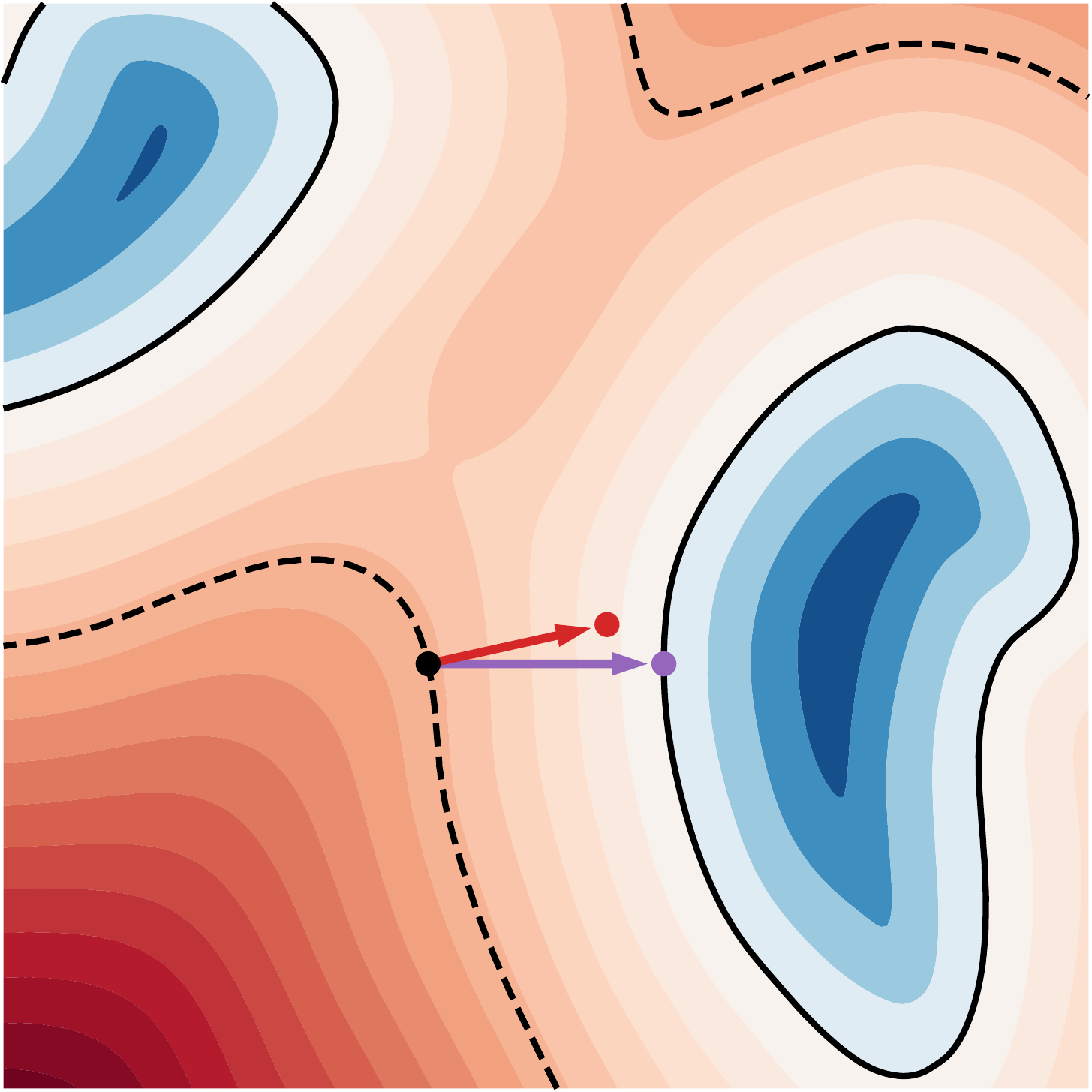}
            \put(32, 65){\scalebox{0.8}{$p_i$}}
            \put(25, 41){\scalebox{0.8}{$x^*$}}
            \put(46, 44){\scalebox{0.8}{$p_i\!-\!fg$}}
        \end{overpic}}          
    \caption{Closest point $x^*$ compared with the SDF approximation. Due to the inaccurate level set (dashed line) at the sample point $p_i$, the gradient $g$ does not point towards the closest point on the surface (full line).}
    \label{fig:approx_closest_point}
\end{wrapfigure}
In order to compute \cref{eqn:s2p-loss}, we also need an efficient method for uniform sampling of points from $\surface_\theta$ in batches. To this end, we build on the approach from \citet{Yang2021GeometryProcessing}.
It operates by first sampling a large number of points $p_i$ uniformly from $\Omega$ and rejecting the points if $|f(\theta, p_i)| > \tau$.
Then, it generates the surface samples by computing 
\eq{
    \label{eqn:sdf-pull}
    x_i(\theta) = p_i - f(\theta, p_i)\grad f(\theta, p_i),
}
which equals to the closest point of $p_i$ on $\surface_\theta$ for the case when $f(\theta, \argdot)$ is an SDF.
However, sampling $\Omega$ densely enough to obtain a sufficient amount of points on $S_{\theta_k}$ at each iteration $k$ of network training is costly.
To improve efficiency, we leverage the observation that $\surface_{\theta_\trainIter}$ changes gradually during training.
We extract a mesh $\mesh_\trainIter$ every $\meshInterval$ iterations using marching cubes, which approximates $\surface_{\theta_\trainIter}$.
On this mesh $\mesh_k$, we uniformly sample a large bank of points $\meshSamples$.
It is from this $\meshSamples$, rather than $\Omega$, that we sample points $p_i$ uniformly in \cref{eqn:sdf-pull} for the next $\meshInterval$ training iterations.
This approach greatly decreases the number of rejected points at each iteration.
We use $\meshInterval\equal1000$ to strike the balance between the computational cost of marching cubes and the approximation quality of the implicit surface.

We further account for the fact that $f(\theta, \argdot)$ may not be a good approximation of an SDF everywhere in $\Omega$.
Using \cref{eqn:sdf-pull} when $f(\theta, x)$ is not an SDF can produce points that are far from the zero-level set $\surface_\theta$, as demonstrated in \cref{fig:approx_closest_point}.
As a solution, we propose to replace the projection in \cref{eqn:sdf-pull} with \textit{SDF-descent} by setting $x_i(\theta) = x_i^M$, where
\eq{
\label{eqn:sdf-descent}
    x_i^{m+1} = x_i^m - f(\theta, x_i^m)\frac{\grad f(\theta, x_i^m)}{\|\grad f(\theta, x_i^m)\|},
    \quad
    x_i^0 := p_i,
}
which will take multiple steps using the SDF approximation.
Note that \cref{eqn:sdf-descent} is equivalent to \cref{eqn:sdf-pull} when $f(\theta, \argdot)$ is an SDF.
We reject any sample point with $|f(\theta, x_i^M)| > \epsilon$, which indicates that SDF-descent did not converge to a point on the implicit surface.
In our experiments, we set $\epsilon = 0.001$ and $M=4$.

\section{The SIREN loss -- a surface area regularizer}
\label{sec:ssa-surface-area-regularizer}
The off-surface term from \siren~\cite{sitzmann2019siren} is a widely used alternative to mitigating spurious surfaces.
Here, we will examine this loss term and establish its link to the surface area, which explain some undersirable side-effects of this regularization in surface reconstruction.

Let us begin by recalling the regularization term from \siren~\cite{sitzmann2019siren}:\footnote{Our definition varies slightly from \citet{sitzmann2019siren} in order to simplify the upcoming analysis: \textit{(i)} We scale the loss by a factor $\alpha$/2, and \textit{(ii)} integrate over $\Omega$ instead of $\Omega\setminus\surface$, which better models the sampling-based implementation.}
\eq{
\label{eqn:siren-term}
    \SSAloss^\alpha(\theta) = \frac{\alpha}{2}\int_{\Omega}e^{-\alpha|f(\theta, x)|}\diff x,
}
which penalizes having distance field values close to zero uniformly in $\Omega$.
\cref{eqn:siren-term} is commonly referred to as an off-surface or non-manifold loss \cite{BenShabat2021DiGS, zixiong23neuralsingular}. However, our upcoming analysis reveals that $\SSAloss^\alpha$, in fact, behaves as a \textit{surface area regularizer}, when it is combined with the eikonal loss.
For this reason, we opt to refer to it as the \textit{SIREN Surface Area} (SSA) loss.
The full SIREN loss without normals is given by:
\eq{
\label{eqn:siren-loss}
    \IGRSSAloss^\alpha(\theta) = \frac{1}{n}\sum_{i=1}^n|f(\theta, x_i)| + \lambda\Eikonalloss(\theta) + \tilde\mu\SSAloss^\alpha(\theta),
}
where $\tilde\mu$ is a weighting parameter.
To ensure consistency of the weighting parameter with \siren~\cite{sitzmann2019siren}, we will from now on work with the rescaled weight $\mu = \tilde\mu|\Omega|\alpha/2$.

Looking closer at \cref{eqn:siren-term}, the value of the penalty term $e^{-\alpha|f(\theta, x)|}$ will be close to 1 near the surface, and close to 0 otherwise,
with a cut-off distance at roughly $1/\alpha$ from the surface.
We might therefore already expect the integral to be proportional to the area of the implicit surface, denoted by $|\surface_\theta|$.
The following theorem will confirm this intuition by reinterpreting $\SSAloss^\alpha(\theta)$ as a surface integral. In fact, the value of $\SSAloss^\alpha(\theta)$ approaches $|\surface_\theta|$ in the limit of large $\alpha$ -- at least in the case when $f(\theta, \argdot)$ is an SDF.

\begin{theorem}
    \label{thm:ssa}
    If $\SSAloss^\alpha(\theta)$ is defined as in \cref{eqn:siren-term} and $g(\theta, x) = \grad f(\theta, x)$, then $\SSAloss^\alpha(\theta)$ approaches the surface integral of $1/\|g(\theta, x)\|$ over the surface $S_\theta$ as $\alpha\rightarrow\infty$. That is:
    \eq{
        \SSAloss^\infty(\theta)
        =
        \lim_{\alpha\rightarrow\infty}\SSAloss^\alpha(\theta)
        =
        \int_{x\in S_\theta}\frac{1}{\|g(\theta, x)\|}dS.
    }
    Additionally, if $f(\theta, \argdot)$ is an SDF, $\SSAloss^\alpha(\theta)$ approaches the surface area of $\surface_\theta$:
    \eq{
        \SSAloss^\infty(\theta) = |\surface_\theta|.
    }
\end{theorem}
\begin{proof}
See \cref{sec:proof}.
The key idea is to recognize that the penalty term $\frac{\alpha}{2}e^{-\alpha|f(\theta, x)|}$ approaches a Dirac delta centered on $\surface_\theta$ as $\alpha\rightarrow\infty$.
This allows us to integrate out one spatial dimension, leaving a surface integral weighted by the factor $1/\|g(\theta, x)\|$. When $f(\theta,\argdot)$ is an SDF, the surface integral reduces to the surface area since $\|g(\theta, x)\| = 1$ for all $x$. \qed
\end{proof}

As a consequence of \cref{thm:ssa}, minimizing $\SSAloss^\alpha(\theta)$ will tend to ``shrink away'' spurious surfaces. But it will also tend to remove details from the actual shape.
If the weight parameter $\mu$ is too large, the zero-level set may even disappear fully.
We demonstrate this effect analytically for a toy example in \cref{sec:toy-example}, and empirically in \cref{sec:ssa-smoothing}.

We remark that \cref{thm:ssa} is a variant of the ``co-area'' formula \cite{Federer1996, MORGAN200021} adapted to the special case of neural implicit surfaces. 
The idea of regularizing surface area using an integral over delta functions dates back to \citet{Chan2001Contours}, where it was used for image segmentation.

\inparagraph{Sample approximation.} In practice, the volume integral in \cref{eqn:siren-term} cannot be computed exactly.
Instead, the loss term is computed using a sample average: 
\eq{
    \label{eqn:ssa-sample}
    \sampleSSAloss^\alpha(\theta)=\frac{|\Omega|}{K}\frac{\alpha}{2}\sum_{i=1}^Ke^{-\alpha|f(\theta, x_i)|},
}
where the points $\{x_i\}_{i=1}^K$ are sampled uniformly from $\Omega$.
We numerically verify that \cref{eqn:ssa-sample} is a good approximation of \cref{eqn:siren-term} in \cref{sec:siren-sample-approximation}.

\section{Experiments}
\label{sec:experiments}

\subsection{Implementation details}
\label{sec:implementation-details}
We compare DiffCD against the loss functions from IGR \cite{gropp2020implicit}, \neupull~\cite{Ma2020NeuralPull} and SIREN \cite{sitzmann2019siren} (without the normal alignment term).
In order to ensure a fair comparison between loss functions, we implement all methods in a common framework using JAX \cite{jax2018github}. We use the same MLP  architecture and training procedure for all methods. See \cref{sec:additional-details} for more details.
We also compare our approach with supervised methods Points2Surf \cite{Erler2020Points2Surf}, POCO \cite{Boulch:2022:POCO} and NKSR \cite{huang2023nksr}. 
For Points2Surf and POCO, we use the checkpoints trained on the ABC dataset \cite{Koch_2019_ABC}.
For NKSR, we also use the model trained on ABC. Since this model requires normals as input, we use the normal-free model trained on ScanNet \cite{dai2017scannet} in order to estimate surface normals.
We note that this leaves NKSR at a disadvantage compared to the other supervised methods, which were trained fully on ABC.
We will denote this variant by NKSR$^\dagger$ in order to emphasize the distinction.

\inparagraph{Parameter settings.} As our base setting, we use $\lambda\equal 0.1$, as in IGR \cite{gropp2020implicit}. We use $\mu\equal 0.033$ to keep the same ratio between the point cloud loss and the SSA loss as in SIREN \cite{sitzmann2019siren}.
We also evaluate SIREN with $\mu=0.33$ since we find that the original value is too small to remove the most prominent spurious surfaces (see \cref{sec:ssa-smoothing}).
As we demonstrate in \cref{sec:eikonal-smoothing}, using larger $\lambda$ produces smoother surfaces and helps avoid wrinkly surfaces that overfit the noise.
We therefore increase $\lambda$ to 0.5 and 1 in the ``Medium Noise'' and ``Max Noise'' settings, respectively, for all methods that use the eikonal loss.

\subsection{Surface reconstruction}
\begin{table}[t]
\scriptsize
\sisetup{detect-weight=true,detect-inline-weight=math}
\setlength\tabcolsep{4pt}
\caption{\textbf{Reconstruction accuracy on FAMOUS.} Lower is better. The best-performing method in each category is marked in bold. $(^{\dagger})$ indicates that the normal estimator of NKSR is trained on 
 ScanNet. DiffCD consistently outperforms other optimization-based methods and remains competitive with supervised methods.}
\begin{tabularx}{\linewidth}{
X
S[table-format=1.3, table-column-width=3.1em]@{\hspace{0.1em}}
S[table-format=1.3, table-column-width=3.1em]@{\hspace{0.1em}}
S[table-format=2.1,      table-column-width=3.1em]
S[table-format=1.3, table-column-width=3.1em]@{\hspace{0.4em}}
S[table-format=1.3, table-column-width=3.1em]@{\hspace{0.1em}}
S[table-format=2.1,      table-column-width=3.1em]
S[table-format=1.3, table-column-width=3.1em]@{\hspace{0.8em}}
S[table-format=1.3, table-column-width=3.1em]@{\hspace{0.1em}}
S[table-format=2.1,      table-column-width=3.1em]
}
\toprule
        & \multicolumn{3}{c}{No Noise} 
        & \multicolumn{3}{c}{Medium Noise} 
        & \multicolumn{3}{c}{Max Noise} \\
        \cmidrule(lr){2-4}\cmidrule(lr){5-7}\cmidrule(lr){8-10}
  Methods
  & {CD} & {CD$^2$}  & {CA ($^\circ$)} 
  & {CD} & {CD$^2$}  & {CA ($^\circ$)} 
  & {CD} & {CD$^2$}  & {CA ($^\circ$)} 
  \\
\midrule
\multicolumn{10}{l}{\textit{Supervised}} \\
\midrule[0.1pt]
Points2Surf~\cite{Erler2020Points2Surf}
&  0.589 &  1.099 & 26.8 
&  0.652 &  1.261 & 30.1
&  \textbf{1.151} &  \textbf{2.786} & 47.4
\\

POCO~\cite{Boulch:2022:POCO}
&  \textbf{0.539} &  0.646 & \textbf{20.9} 
&  0.642 &  0.932 & 25.8
&  1.385 &  4.945 & 44.7
\\

NKSR$^{\dagger}$~\cite{huang2023nksr}
&  \textbf{0.539} &  \textbf{0.519} & 21.4
&  \textbf{0.610} &  \textbf{0.642} & \textbf{24.7}
&  1.636 &  5.298 & \textbf{44.6}
\\

\midrule[0.5pt]
\multicolumn{10}{l}{\textit{Optimization-based}} \\
\midrule[0.1pt]
Neural-Pull~\cite{Ma2020NeuralPull}
& 0.641 & 1.851 & 21.1 
& 0.924 & 4.086 & 41.7 
& 1.926 & 11.114 & 44.9 
\\

IGR~\cite{gropp2020implicit}
& 0.772 & 3.092 & 23.0 
& 1.454 & 13.185 & 49.8 
& 2.553 & 27.912 & 45.7 
\\

SIREN~\cite{sitzmann2019siren} $(\mu\equal 0.033)$
& 0.602 & 1.151 & 22.1 
& 0.963 & 3.647 & 32.5 
& 2.079 & 14.268 & 44.9 
\\

SIREN~\cite{sitzmann2019siren} $(\mu\equal 0.33)$
& 0.605 & 1.525 & 19.2 
& 0.739 & 1.724 & 26.8 
& 1.658 & 7.698 & 41.1 
\\

DiffCD (ours)
& \textbf{0.518} & \textbf{0.542} & \textbf{17.5} 
& \textbf{0.690} & \textbf{1.069} & \textbf{26.6} 
& \textbf{1.416} & \textbf{4.306} & \textbf{37.4} 
\\
\bottomrule
\end{tabularx}

\label{table:famous_results}
\end{table}

\begin{figure}[t]
    \begin{subfigure}[t]{0.03\columnwidth}
        \rotatebox[origin=l]{90}{\makebox[0.5cm]{Max \qquad  Medium}}%
    \end{subfigure}%
    \hfill
    \begin{subfigure}[t]{0.19\columnwidth}
        \includegraphics[width=\columnwidth,trim={0 0 0 10mm},clip]{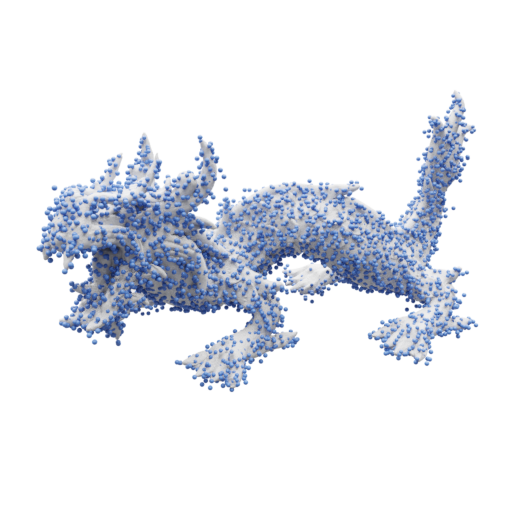}\\[-5mm]
        \includegraphics[width=\columnwidth,trim={0 30mm 0 0},clip]{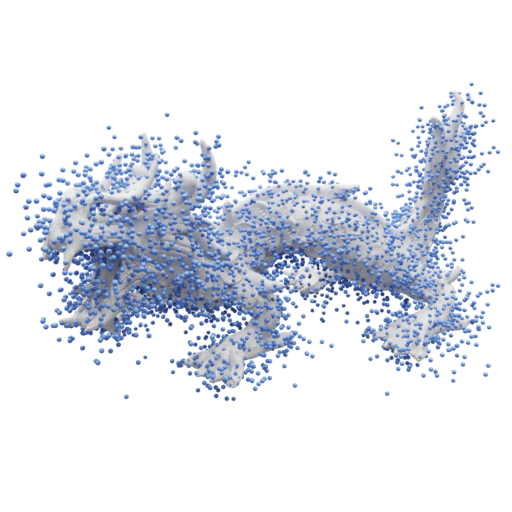}
        \caption{Ground truth}
        \label{subfig:noisy-gt}
    \end{subfigure}%
    \hfill
    \begin{subfigure}[t]{0.19\columnwidth}
        \includegraphics[width=\columnwidth,trim={0 0 0 10mm},clip]{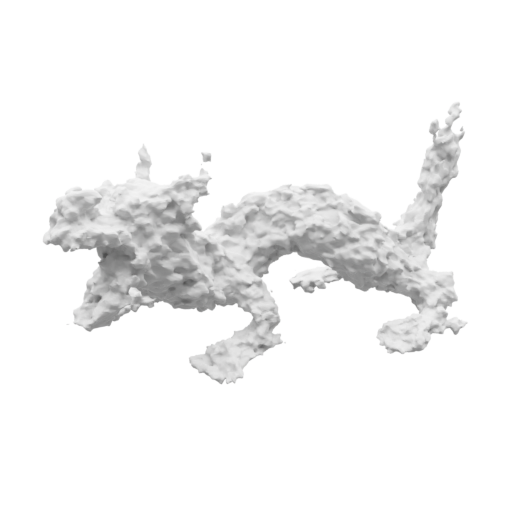}\\[-5mm]
        \includegraphics[width=\columnwidth,trim={0 30mm 0 0},clip]{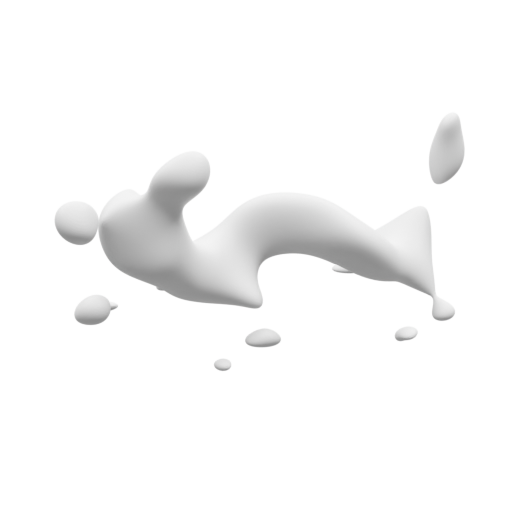}
        \caption{Neural-Pull \cite{Ma2020NeuralPull}}
        \label{subfig:noisy-neuralpull}
    \end{subfigure}%
    \hfill
    \begin{subfigure}[t]{0.19\columnwidth}
        \includegraphics[width=\columnwidth,trim={0 0 0 10mm},clip]{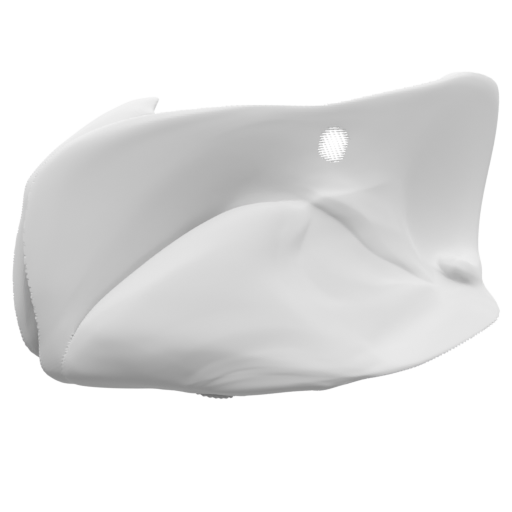}\\[-5mm]
        \includegraphics[width=\columnwidth,trim={0 30mm 0 0},clip]{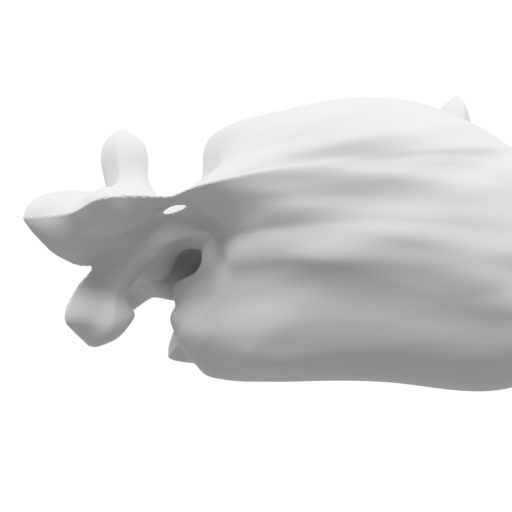}
        \caption{IGR \cite{gropp2020implicit}}
        \label{subfig:noisy-igr}
    \end{subfigure}%
    \hfill
    \begin{subfigure}[t]{0.19\columnwidth}
        \includegraphics[width=\columnwidth,trim={0 0 0 10mm},clip]{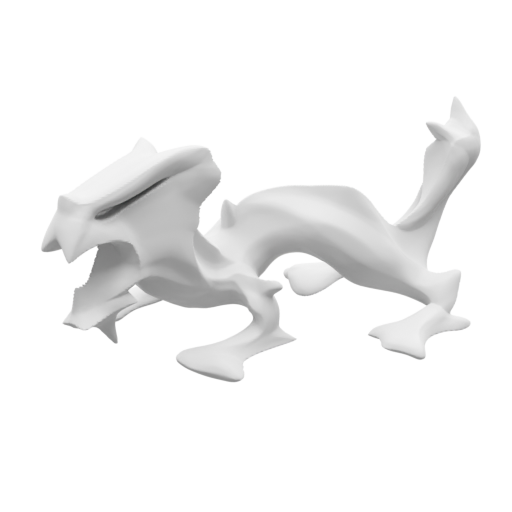}\\[-5mm]
        \includegraphics[width=\columnwidth,trim={0 30mm 0 0},clip]{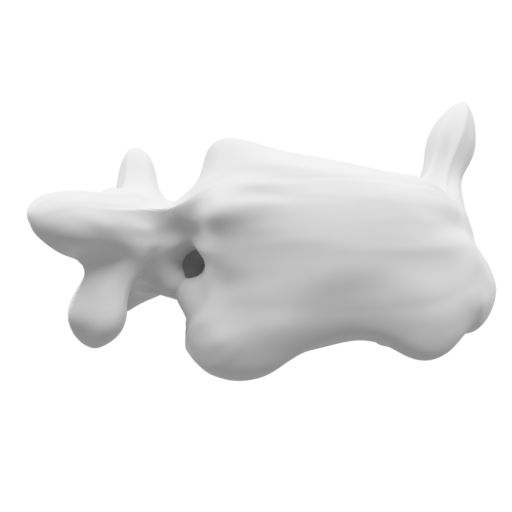}
        \caption{SIREN \cite{sitzmann2019siren}}
        \label{subfig:noisy-siren}
    \end{subfigure}%
    \hfill
    \begin{subfigure}[t]{0.19\columnwidth}
        \includegraphics[width=\columnwidth,trim={0 0 0 10mm},clip]{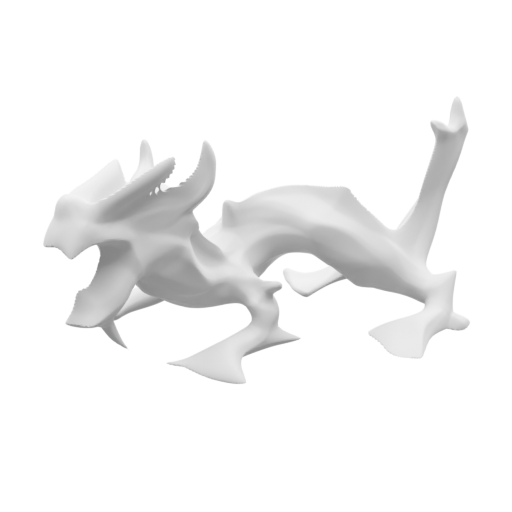}\\[-5mm]
        \includegraphics[width=\columnwidth,trim={0 30mm 0 0},clip]{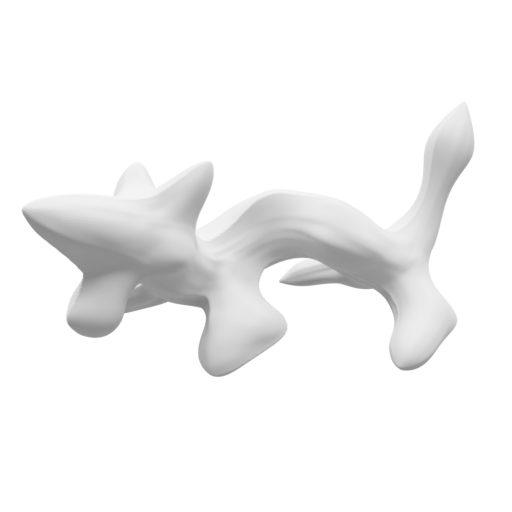}
        \caption{DiffCD (ours)}
        \label{subfig:noisy-diffcd}
    \end{subfigure}

    \caption{
        \textbf{Surface reconstruction from noisy point clouds.}
        \subref{subfig:noisy-gt}) Ground truth surface and noisy point clouds.
        \subref{subfig:noisy-neuralpull}) Neural-Pull is highly sensitive to noise, producing irregular surfaces with missing chunks.
        For \subref{subfig:noisy-igr}) IGR, the issue of spurious surfaces remains in the noisy scenarios, and
        \subref{subfig:noisy-siren}) SIREN ($\mu=0.33$) again ends up smoothing out surface details.
        \subref{subfig:noisy-diffcd}) DiffCD recovers the general shape despite the extreme noise level.
    }
    \label{fig:noisy-shapes-comparison}
\end{figure}

\begin{figure}[t]
    \centering
    \tiny
    \begin{subfigure}{0.49\linewidth}
        \begin{minipage}[t]{0.24\linewidth}
            \centering
            \includegraphics[width=\columnwidth]{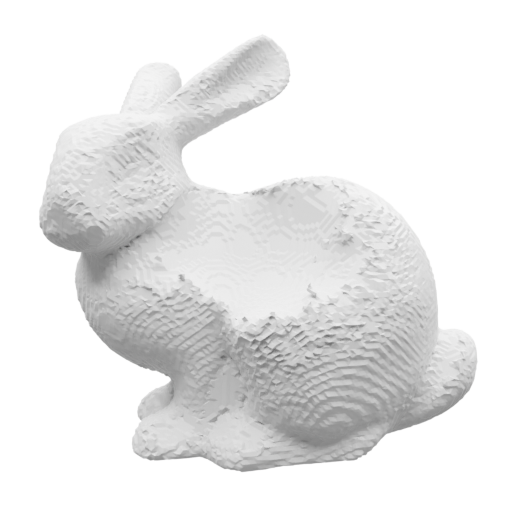}\\
            \includegraphics[width=\columnwidth]{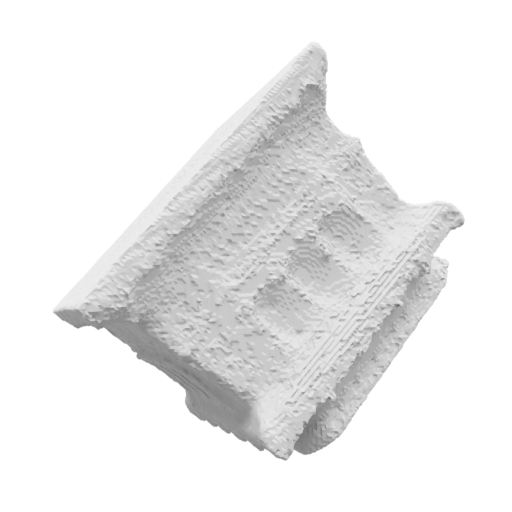}\\
            Points2Surf
        \end{minipage}%
        \begin{minipage}[t]{0.24\linewidth}
            \centering
            \includegraphics[width=\columnwidth]{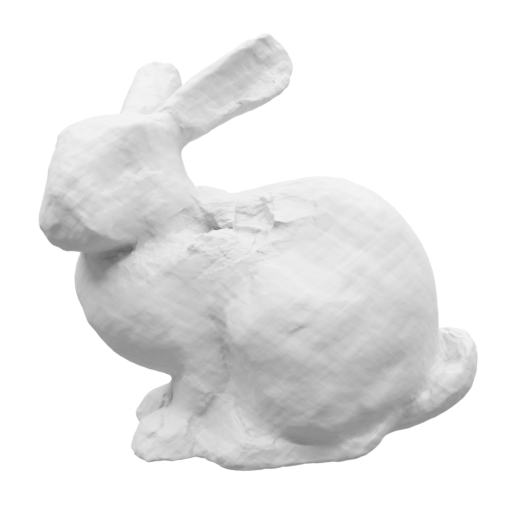}\\
            \includegraphics[width=\columnwidth]{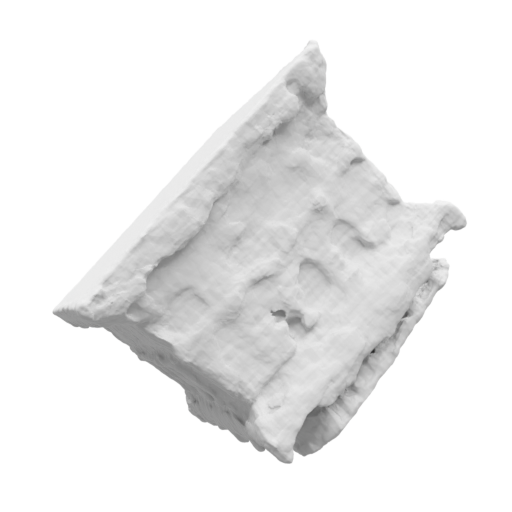}\\
            \quad POCO
        \end{minipage}%
        \begin{minipage}[t]{0.24\linewidth}
            \centering
            \includegraphics[width=\columnwidth]{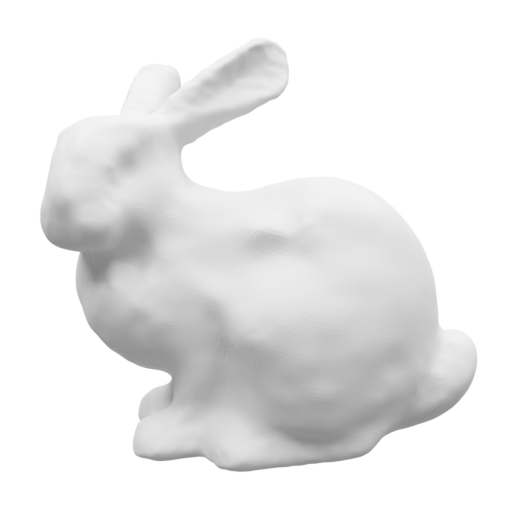}\\
            \includegraphics[width=\columnwidth]{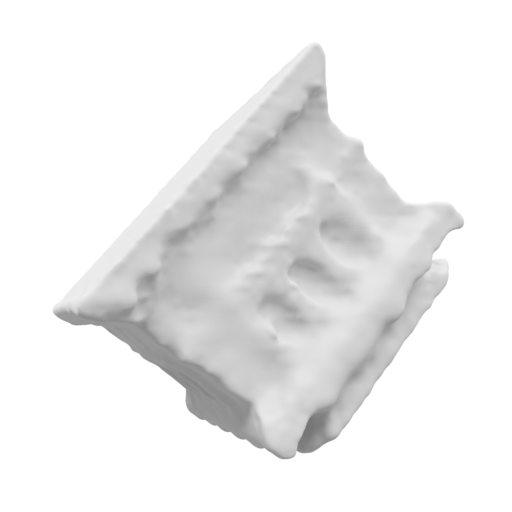}\\
            \quad NKSR$^\dagger$
        \end{minipage}%
        \begin{minipage}[t]{0.24\linewidth}
            \centering
            \includegraphics[width=\columnwidth]{figures/diffcd/bunny.png}\\
            \includegraphics[width=\columnwidth]{figures/diffcd/LibertyBase.png}\\
            \quad DiffCD
        \end{minipage}
        \caption{No noise}
    \end{subfigure}%
    \hfill
    \begin{subfigure}{0.49\linewidth}
        \begin{minipage}[t]{0.24\linewidth}
            \centering
            \includegraphics[width=\columnwidth]{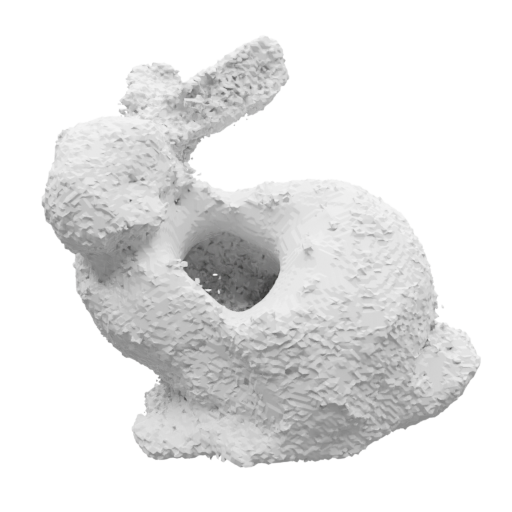}\\
            \includegraphics[width=\columnwidth]{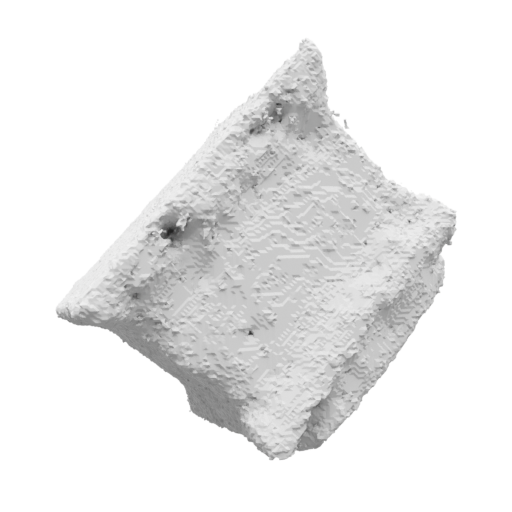}\\
            Points2Surf
        \end{minipage}%
        \begin{minipage}[t]{0.24\linewidth}
            \centering
            \includegraphics[width=\columnwidth]{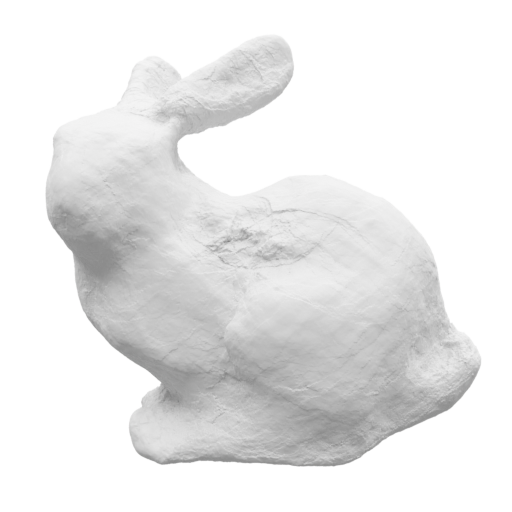}\\
            \includegraphics[width=\columnwidth]{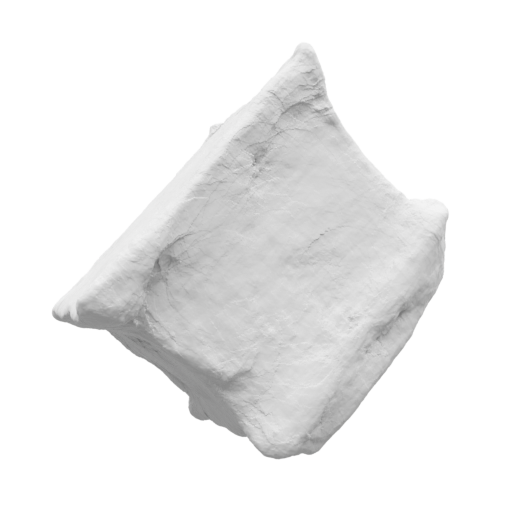}\\
            \quad POCO
        \end{minipage}%
        \begin{minipage}[t]{0.24\linewidth}
            \centering
            \includegraphics[width=\columnwidth]{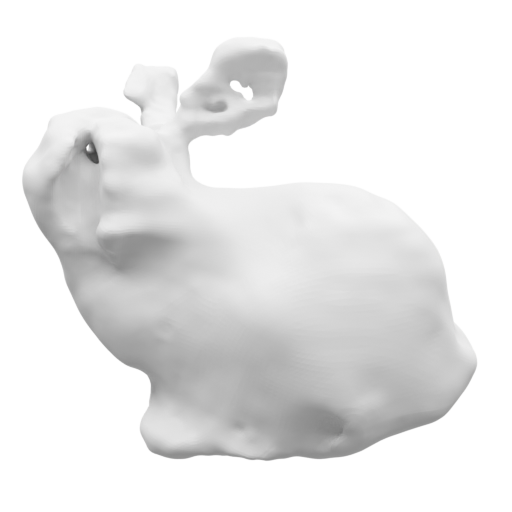}\\
            \includegraphics[width=\columnwidth]{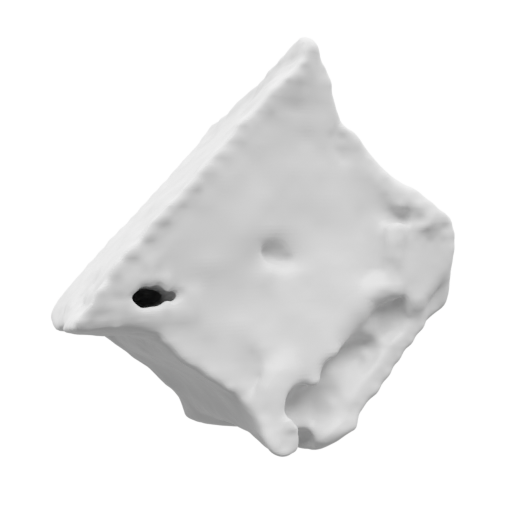}\\
            \quad NKSR$^\dagger$
        \end{minipage}%
        \begin{minipage}[t]{0.24\linewidth}
            \centering
            \includegraphics[width=\columnwidth]{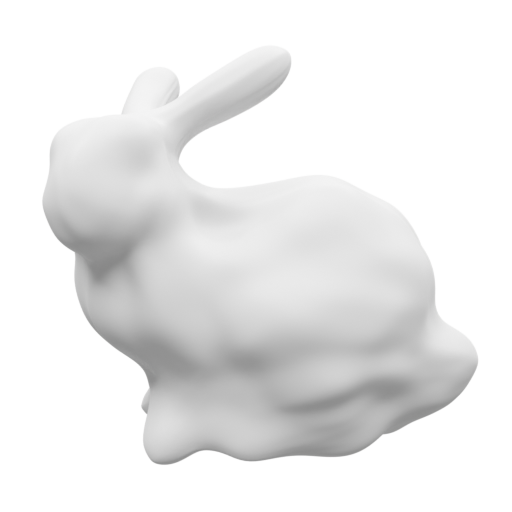}\\
            \includegraphics[width=\columnwidth]{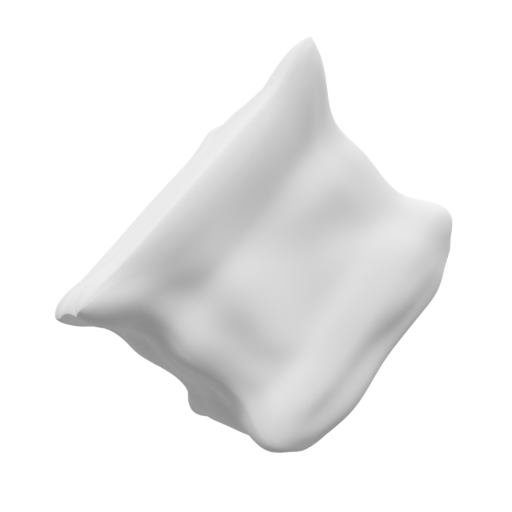}\\
            \quad DiffCD
        \end{minipage}
        \caption{Max noise}
    \end{subfigure}

    \caption{\textbf{Comparison to supervised methods.} Our method remains competitive with supervised methods. Particularly in the noise-free setting, where DiffCD recovers higher-quality surface detail.}
    \label{fig:supervised-vs-diffcd}
\end{figure}

\begin{table}[t]
\begin{minipage}[t]{0.44\textwidth}
\vspace{0pt}
\centering
\scriptsize
\sisetup{detect-weight=true,detect-inline-weight=math}
\setlength\tabcolsep{4pt}
\centering
\caption{\textbf{Reconstruction accuracy on Thingi10k.} Lower is better. The best performing method in each category is marked in bold.}
\label{table:thingi10k_results}
\begin{tabularx}{\linewidth}{
X
S[table-format=1.3, table-column-width=3.1em]@{\hspace{0.1em}}
S[table-format=1.3, table-column-width=3.1em]@{\hspace{0.1em}}
S[table-format=2.1,      table-column-width=3.1em]
S[table-format=1.3, table-column-width=3.1em]@{\hspace{0.4em}}
S[table-format=1.3, table-column-width=3.1em]@{\hspace{0.1em}}
S[table-format=2.1,      table-column-width=3.1em]
S[table-format=1.3, table-column-width=3.1em]@{\hspace{0.8em}}
S[table-format=1.3, table-column-width=3.1em]@{\hspace{0.1em}}
S[table-format=2.1,      table-column-width=3.1em]
}
\toprule
  Mehtod
  & {CD} & {CD$^2$}  & {CA ($^\circ$)} 
  \\
\midrule
\multicolumn{4}{l}{\textit{Supervised}} \\
\midrule[0.1pt]
Points2Surf
&  0.519 &  0.639 & 19.5
\\

POCO
&  0.534 &  1.100 & 14.4
\\

NKSR$^{\dagger}$
&  \textbf{0.497} &  \textbf{0.463} & \textbf{13.0}
\\
\midrule[0.5pt]
\multicolumn{4}{l}{\textit{Optimization-based}} \\
\midrule[0.1pt]
Neural-Pull
& 0.688 & 2.604 & 15.9 
\\

IGR
& 0.868 & 7.018 & 16.4
\\

SIREN $\mu\equal 0.033$
& 0.564 & 1.042 & 14.2 
\\

SIREN $\mu\equal 0.33$
& 0.538 & 0.894 & \textbf{12.7} 
\\

DiffCD (ours)
& \textbf{0.500} & \textbf{0.580} & 13.0 
\\
\bottomrule
\end{tabularx}
\end{minipage}%
\hfill
\begin{minipage}[t]{0.53\textwidth}
\vspace{-0.5em}
\centering
\begin{subfigure}{0.33\columnwidth}
    \includegraphics[width=\columnwidth]{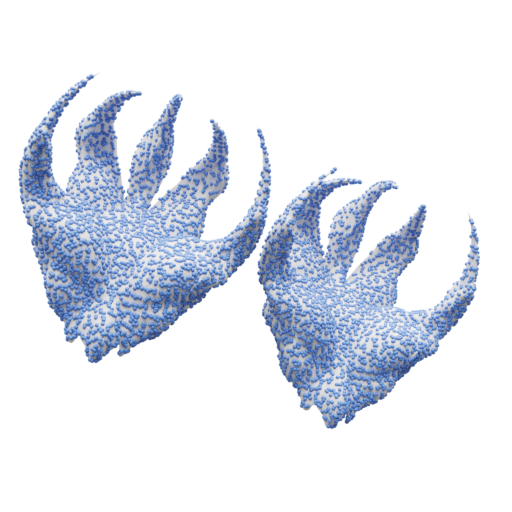}
    \caption{Ground truth}
\end{subfigure}%
\begin{subfigure}{0.33\columnwidth}
    \includegraphics[width=\columnwidth]{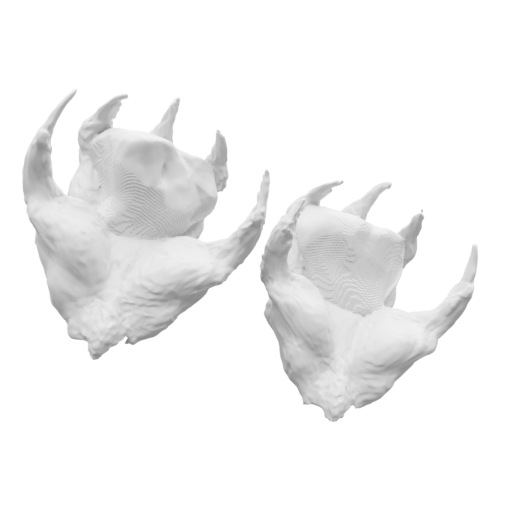}
    \caption{Neural-Pull}
\end{subfigure}%
\begin{subfigure}{0.33\columnwidth}
    \includegraphics[width=\columnwidth]{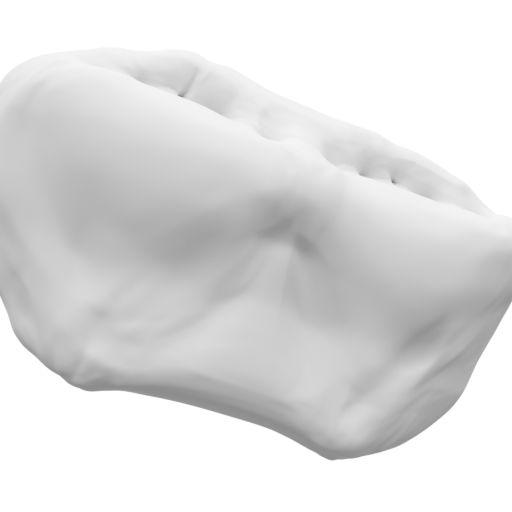}
    \caption{IGR}
\end{subfigure}\\
\begin{subfigure}{0.33\columnwidth}
    \includegraphics[width=\columnwidth]{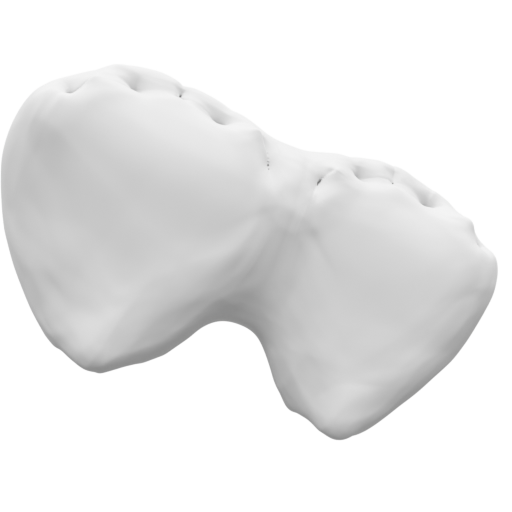}
    \caption{Neural-Pull}
\end{subfigure}%
\begin{subfigure}{0.33\columnwidth}
    \includegraphics[width=\columnwidth]{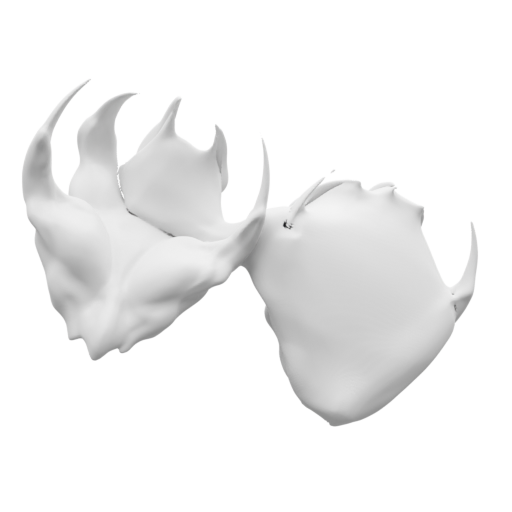}
    \caption{IGR}
\end{subfigure}%
\begin{subfigure}{0.33\columnwidth}
    \includegraphics[width=\columnwidth]{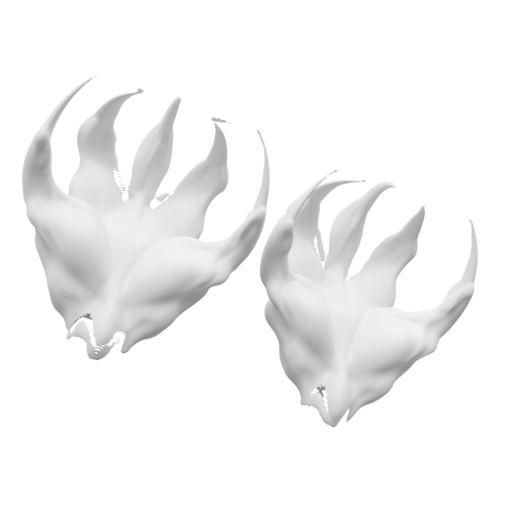}
    \caption{DiffCD (ours)}
\end{subfigure}
\captionof{figure}{\textbf{Thingis10k -- qualitative results.}}
\label{fig:thingi10k_qual}

\end{minipage}

\end{table}

We test our method on the FAMOUS dataset \cite{Erler2020Points2Surf}.
The dataset consists of 22 ground-truth meshes representing the true surfaces $\surface$, each with three simulated point cloud scans: ``No Noise'', ``Medium Noise'' and ``Max Noise''. 
The point density varies significantly across each shape due to the realistic scanning procedure.
Several shapes also have chunks of missing points, such as on the bunny in \cref{fig:spurious-igr}. 
Additionally, we test our method on the ``No Noise'' variant of the larger Thingi10k test set, also prepared by \citet{Erler2020Points2Surf}. It consists of 100 point clouds obtained with the same scanning procedure.
On average, we find that DiffCD takes roughly 6$\times$ longer for the same number of iterations compared to IGR \cite{gropp2020implicit}, which is the fastest optimization-based method.
The most time-consuming part is the sampling of $\surface_\theta$; improving this aspect is a promising direction for future work.

\inparagraph{Metrics.}
After fitting each neural implicit surface, we extract a mesh using marching cubes \cite{lewiner2003efficientMarchingCubes} with a resolution of 512 and evaluate the following shape metrics:
\textit{(i)} Chamfer distance (CD), as defined in \cref{eqn:p2p}.
We use 30K sample points from both $\surface$ and $\surface_\theta$.
\textit{(ii)} Chamfer squared distance (CD$^2$), which is similar to the Chamfer distance, except for computing the average \textit{squared} distances from $\pointcloud_1$ to $\pointcloud_2$ and vice versa. In contrast to CD, CD$^2$ is more sensitive to outlier surfaces.
\textit{(iii)} Chamfer angle (CA) is the average angle in degrees between the surface normal at each sample point in $\surface$ and its closest point in $\surface_\theta$, and vice versa. In order to take surface orientation into account, we compute CA using $\surface$ with and without flipped normals, and pick the one yielding the smaller value.
We scale CD and CD$^2$ up by a factor of 100 and 100$^2$ each for readability.

\inparagraph{Results.}
We present our main results on the FAMOUS dataset in \cref{table:famous_results}.
On this dataset, DiffCD decisively outperforms the optimization-based baselines \emph{on all metrics and in all settings of the noise level}.
We highlight 4 reconstructed shapes in the ``No Noise'' setting in \cref{fig:spurious-igr}. While IGR \cite{gropp2020implicit} can obtain great surface detail in parts, its performance is significantly deteriorated by spurious surfaces in several shapes. SIREN \cite{sitzmann2019siren} mitigates the spurious surfaces to some extent but tends to either under-smooth, leaving some spurious surfaces behind (third row), or over-smooth, erasing surface details (remaining rows).
In contrast, DiffCD cleanly removes all spurious surfaces without sacrificing surface detail.

Notably, our method outperforms the supervised methods Points2Surf and POCO in the noise-free setting and compares favorably with NKSR on CD and CA.
Despite being trained on point clouds from the same scanning procedure, we find that supervised methods tend to suffer from a lack of surface detail, as demonstrated in \cref{fig:supervised-vs-diffcd}.
DiffCD also remains competitive in the ``Medium'' and ``Max'' noise settings, where other optimization-based methods again struggle with spurious surfaces, as can be seen in \cref{fig:noisy-shapes-comparison}.

Finally, \cref{table:thingi10k_results} shows the results on Thingi10k and example reconstruction in \cref{fig:thingi10k_qual}.
We find that there are generally less spurious surfaces in this dataset, so the difference between the methods is smaller.
Nevertheless, our method obtains a $7.2\%$ and $42.6\%$ relative improvement in CD and CD$^2$ over the next-best performing method (SIREN with $\mu\equal0.33$), at a little cost of $2.3\%$ relative increase in CA.

\subsection{Smoothing effect of $\SSAloss^\alpha$}
\label{sec:ssa-smooting}
\begin{figure}[t]%
    \hspace{.005\columnwidth}
    \begin{overpic}[width=.99\columnwidth]{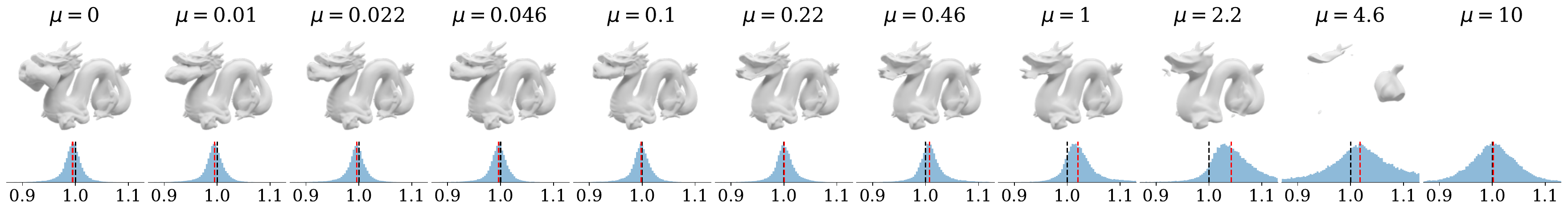}
        \put(-0.2, 8){\rotatebox{90}{\makebox(0,0){\scalebox{0.5}{Dragon}}}}
        \put(-0.2, 3.5){\rotatebox{90}{\makebox(0,0){\scalebox{0.5}{$\|g\|$}}}}
    \end{overpic}

    \hspace{.005\columnwidth}
    \begin{overpic}[width=.99\columnwidth,]{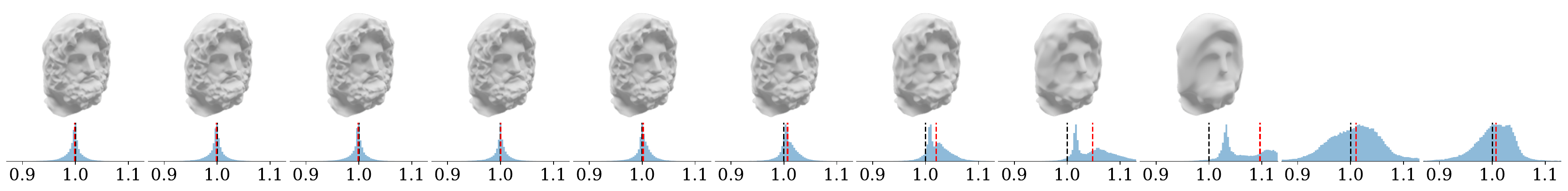}
        \put(-0.2, 8){\rotatebox{90}{\makebox(0,0){\scalebox{0.5}{Serapis}}}}
        \put(-0.2, 3.5){\rotatebox{90}{\makebox(0,0){\scalebox{0.5}{$\|g\|$}}}}
    \end{overpic}
    \vspace{-.7cm}
    \caption{
    Illustration of shapes reconstructed using the SIREN loss for various SSA weight factors $\mu$.
    \textbf{Odd rows:} Reconstructed shapes. \textbf{Even rows:} gradient norm histograms from points sampled near $\pointcloud$. The median gradient norm is marked in red.
    Larger $\mu$ tends to shrink the surface and smooth out details, while also increasing the gradient norm near the surface. If $\mu$ is too large, the shape disappears completely.
    }
    \label{fig:mu-grid}
\end{figure}

\begin{figure}[t]%
    \begin{overpic}[width=\columnwidth]{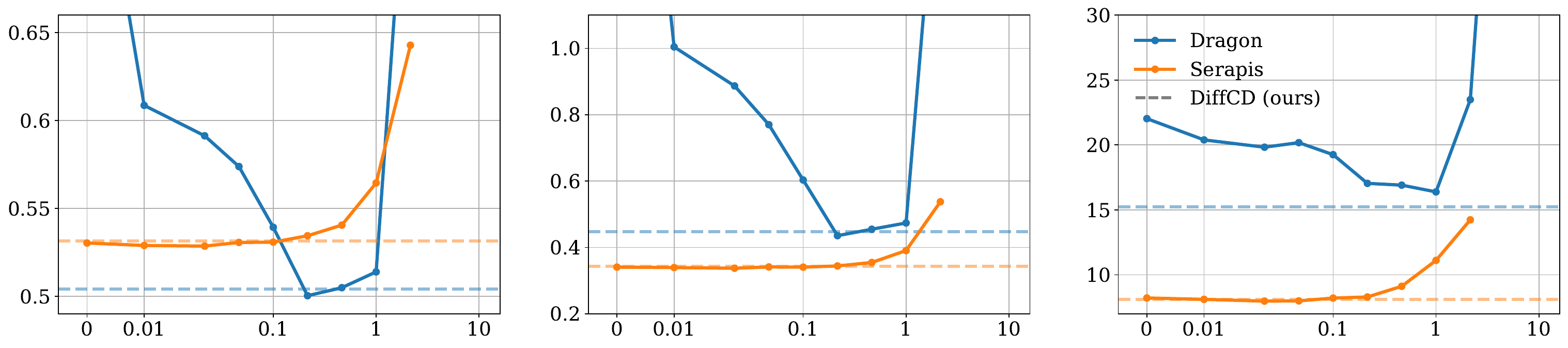}
        \put(16,22){CD}
        \put(16,-1){$\mu$}
        \put(49.5,22){CD$^2$}
        \put(49.5,-1){$\mu$}
        \put(83.5,22){CA}
        \put(83.5,-1){$\mu$}
    \end{overpic}
    \vspace{-0.5cm}
    \caption{
    Shape metrics for SIREN \cite{sitzmann2019siren} from the reconstructions in \cref{fig:mu-grid}. Metrics obtained by DiffCD are marked as dashed lines. While it is possible to carefully tune the weighting factor $\mu$ for each shape separately to obtain similar metrics as DiffCD, the optimal value differs significantly between the two shapes.
    }
    \label{fig:mu-plot}
\end{figure}
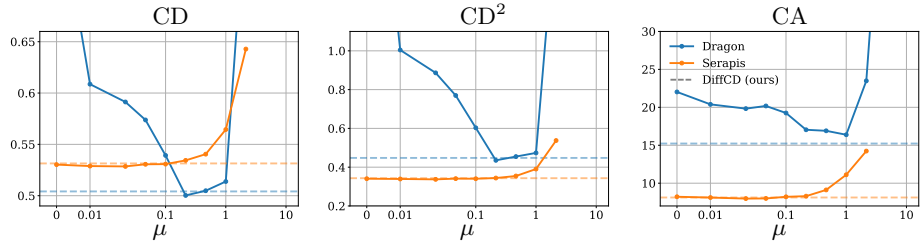
\label{sec:ssa-smoothing}
As we have shown in \cref{sec:ssa-surface-area-regularizer}, the off-surface term $\SSAloss^\alpha$ from SIREN \cite{sitzmann2019siren} approximates the area of the implicit surface, $|\surface_\theta|$.
This means that it will regularize the surface area.
To illustrate this effect in practice, we run SIREN \cite{sitzmann2019siren} for a range of weight parameters $\mu$ between 0 and 10. \cref{fig:mu-grid} shows the results for Dragon and Serapis shapes from the FAMOUS \cite{Erler2020Points2Surf} dataset.
For small $\mu$, the SIREN loss is equivalent to IGR and
the Dragon shape has a large spurious surface in the mouth area.
As we increase $\mu$, the spurious surface gradually shrinks.
However, in the Serapis shape, the details also gradually diminish.
\cref{fig:mu-plot} quantifies this trade-off. Reducing the error in the Dragon shape indeed comes at the cost of increasing the error in the Serapis shape.
In contrast, DiffCD (marked as a dashed line in \cref{fig:mu-plot}) achieves a nearly ideal performance for both shapes without any shape-specific tuning.

Another consequence of \cref{thm:ssa} is that $\SSAloss^\alpha$ can be minimized not only by reducing the surface area, but also by \textit{maximizing} the gradient norm $\|g(\theta, x)\|$ on the surface.
This has the undesirable side-effect of biasing the gradient norm of $f(\theta, \argdot)$ to be greater than 1 near the surface. 
This bias remains even in the limit of an infinitely large $\alpha$.
The histograms in \cref{fig:mu-grid} demonstrate that the bias has a measurable effect on the gradient norms.
For larger values of $\mu$, the median gradient norm increases until the zero-level set disappears from $\Omega$.

\section{Conclusion}
We presented DiffCD, a method for fitting neural implicit surfaces to sparse and noisy point clouds.
We analyzed existing methods and showed that they tend to minimize only one side of the Chamfer distance, which leaves a blind spot for spurious surfaces.
We also showed that the SIREN approach to mitigating spurious surfaces comes at the cost of regularizing the surface area.
By instead minimizing the full symmetric Chamfer distance, our method explicitly mitigates spurious surfaces without being coupled with additional regularization.

\inparagraph{Limitations and future work.}
While DiffCD can be tuned to work well across various noise levels, it does not automatically adapt in cases where the noise level is unknown or varies significantly across the shape.
This stands in contrast to supervised methods, which can handle multiple noise levels using a single model, provided the training data is available.
A promising future direction would be to extend optimization-based methods with learned surface priors, which could help in guiding the optimization process.
We also hope that our analysis of the SIREN \cite{sitzmann2019siren} loss term in \cref{sec:ssa-surface-area-regularizer} can be extended to other kinds of sampling-based surface integrals for implicit surfaces.

\inparagraph{Acknowledgements.} This work was supported by the ERC Advanced Grant SIMULACRON.

\clearpage

%
%
\bibliographystyle{splncs04nat}
\bibliography{main}

\titlerunning{DiffCD: A Symmetric Chamfer Distance for Neural Implicit Surfaces}

\author{Linus Härenstam-Nielsen\orcidlink{0000-0001-6863-4438} \quad
Lu Sang\orcidlink{0009-0007-1158-5584} \quad Abhishek Saroha\\[1mm]
Nikita Araslanov\orcidlink{0000-0002-9424-8837} \quad Daniel Cremers\orcidlink{0000-0002-3079-7984}}

\authorrunning{Härenstam-Nielsen et al.}

\institute{TU Munich \\[1mm]
Munich Center for Machine Learning
}

\title{DiffCD: A Symmetric Differentiable Chamfer Distance for Neural Implicit Surface Fitting\\[2mm]\large -- Supplemental Material --} 
\maketitle

\appendix

\noindent This appendix provides additional material related to our proposed surface reconstruction method, \ourshort, and the experiments. 
In \cref{sec:additional-details}, we expand on the implementation details of DiffCD and our reimplementation of the baseline methods.
In \cref{sec:proof}, we provide the full proof of \cref{thm:ssa}.
In \cref{sec:siren-sample-approximation}, we analyze the impact of $\alpha$ on the sample approximation $\sampleSSAloss^\alpha$.
In \cref{sec:toy-example}, we provide a toy example where the minimizer of the SIREN \cite{sitzmann2019siren} loss has a closed-form solution.
In \cref{sec:eikonal-smoothing}, we provide additional experiments, demonstrating that increasing the eikonal weight results in smoother surfaces.
We also provide additional results showing that all methods -- \neupull, IGR, SIREN and our \ourshort –– benefit from the choice of the parameters used in \cref{sec:experiments}.
Finally, \cref{sec:neural-pull} expands on the comparison between our method and Neural-Pull with a 2D example.

\section{Additional implementation details}
\label{sec:additional-details}
\inparagraph{Network architecture and training.} We use the same coordinate MLP $f(\theta, \argdot)$ for all loss functions, consisting of 8 fully connected layers with 256-dimensional features and softplus activations, using geometric initialization as introduced by \citet{gropp2020implicit}. We also use a skip connection from the input to the 4th layer.
We train each method for 40K iterations of Adam \cite{kingma2014adam} with a learning rate of $10^{-3}$ for the ``No Noise'' and ``Medium Noise'' settings, and $5\!\times\!10^{-5}$ for the ``Max Noise'' setting. We use cosine annealing and a linear warm-up spanning the first 1000 iterations.

\inparagraph{Batch sampling.} Each training batch consists of 5K points $\tilde x_i$ sampled uniformly from the input point cloud $\pointcloud$. These points are used to compute the points-to-surface loss.
For DiffCD, we sample an additional 5K surface points $x_i(\theta)$ from $\surface_\theta$ following our sampling scheme (\cf \cref{sec:diffcd}).
For the eikonal loss, we use a mixed sampling scheme similar to \citet{gropp2020implicit}.
Specifically, we sample 625 ``global'' points uniformly from $\Omega$, along with 5K ``local'' points obtained from a normal distribution centered at each $\tilde x_i$.
The standard deviation of each normal distribution is determined by the distance from $\tilde x_i$ to its $50^\text{th}$ nearest neighbor in $\pointcloud$.
Inspired by the implementation of \citet{Ma2020NeuralPull}, we scale each standard deviation by a factor of $0.2$, as we find that it tends to improve surface detail.
For \neupull \cite{Ma2020NeuralPull}, we only use the 5K local sample points as in the original implementation.
For computing $\sampleSSAloss^\alpha$, we sample 5K points uniformly from $\Omega$ following the original paper.

\inparagraph{Normalization.} Before training, we follow \citet{Ma2020NeuralPull} and compute a \texttt{XYZ}-aligned bounding box around each point cloud, and normalize it by subtracting the box center and scaling by the maximum side length.
As a result, the normalized point cloud fits into a $[-0.5, 0.5]^3$ bounding box.
We de-normalize all shapes to their original size and position before computing metrics.

\section{Proof of Theorem 1}
\label{sec:proof}

\begin{proof}
Note that, in the limit of large $\alpha$, the function $\frac{\alpha}{2}e^{-\alpha|\scalar|}$ approaches the \textit{Dirac delta distribution}\footnote{The delta distribution is an idealized function, defined such that $\delta(\scalar) = 0$ for $\scalar \neq 0$, and $\int_{-\infty}^\infty\delta(\scalar)\diff\scalar = 1$.
For an introduction, see \eg \citet[Chapter 1]{ARFKEN20131}.}
$\delta(\scalar)$. 
Using this fact, we can take the limit of \cref{eqn:siren-term} as  $\alpha\rightarrow\infty$ to obtain:
\eq{
\label{eqn:ssa-inf}
    \SSAloss^\infty(\theta) 
    =
    \lim_{\alpha\rightarrow\infty}\SSAloss^\alpha(\theta)
    = 
    \int_{\Omega}\delta(f(\theta, x))\diff x.
}
At this point, the theorem follows directly from \eg \citet[Theorem 6.1.5]{Hörmander2003}. Nevertheless, we will provide a proof in our notation with the purpose of providing a geometric interpretation.
The main idea is to decompose the volume integral over $\Omega$ into a surface integral over $\surface_\theta$ and a line integral over a short segment along the surface normal $\hat g(\theta, x) = g(\theta, x) / \|g(\theta, x)\|$:
\eq{
\label{eqn:SSA-inf}
    \int_\Omega\delta(f(\theta, x))\diff x 
    &=
    \int_{x\in\surface_\theta}\int_{-\epsilon}^\epsilon\delta(f(\theta, x + \scalar\hat g(\theta, x)))\diff\scalar\diff S.
}
This decomposition is valid since the surface normal is orthogonal to the surface element $\diff S$ everywhere on the surface.
The particular value of $\epsilon$ does not affect the result as long as it is chosen small enough to ensure there are no additional 0-level crossings of $f(x + \tau\hat g(\theta, x))$ on the interval $\tau\in [-\epsilon, \epsilon]$.

We then use a property of the Dirac delta, namely that $\delta(h(\scalar)) = \delta(\scalar)/h'(0)$ for any function $h:\mathbb{R}\rightarrow\mathbb{R}$ whose only zero occurs at $\scalar = 0$. In our case, we set $h_x(\tau) = f(\theta, x + \scalar\hat g(\theta, x))$, which satisfies $h_x(0) = f(\theta, x) = 0$ and $h_x'(0) \equal \|g(\theta, x)\|$, for all $x$ on $\surface_\theta$.
We can then rewrite \cref{eqn:SSA-inf} in terms of $h_x$ to get:
\eq{
    \SSAloss^\infty(\theta)
    &=\int_\Omega\delta(f(\theta, x))\diff x 
\\  &=\int_{x\in\surface_\theta}\int_{-\epsilon}^\epsilon h_x(\tau)\diff\scalar\diff S
\\  &=\int_{x\in\surface_\theta}\int_{-\epsilon}^\epsilon\frac{\delta(\scalar)}{\|g(\theta, x)\|}\diff\scalar\diff S
\\  &=\int_{x\in\surface_\theta}\int_{-\epsilon}^\epsilon\delta(t)\diff t\frac{1}{\|g(\theta, x)\|}\diff S
\\  &=\int_{x\in\surface_\theta}\frac{1}{\|g(\theta, x)\|}\diff S,
}
which completes the first part of the proof. If $f(\theta, x)$ is an SDF, we additionally have $\|g(\theta, x)\| = 1$ for all $x\in\surface_\theta$, so:
\eq{
    \SSAloss^\infty(\theta)
    =\int_{x\in\surface_\theta}\diff S
    = |\surface_\theta|,
}
which completes the second part of the proof. \qed
\end{proof}

\section{Siren loss sample approximation}
\label{sec:siren-sample-approximation}
\begin{figure}[t]
    \newcommand{\figheight}{.27\columnwidth}
    \newcommand{\histratio}{777 / 669}
    \hspace{0.05cm}
    \begin{overpic}[
        height=\figheight, trim={.3cm 1cm 0 1cm}, clip,
    ]{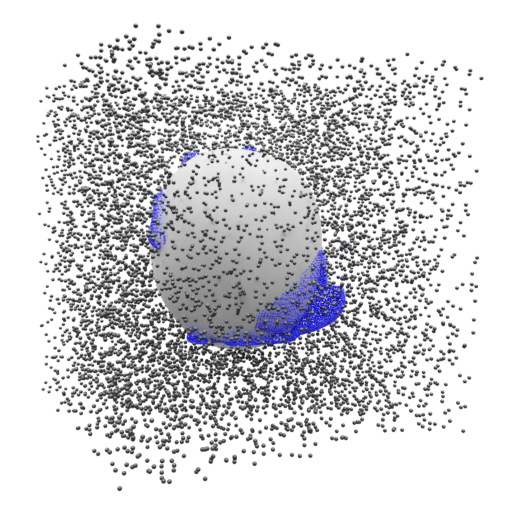}
        \put(0, 50){\makebox(0,0){\rotatebox{90}{Initial $S_\theta$}}}
    \end{overpic}
    \hspace{-0.3cm}
    \begin{overpic}[
        height=\figheight,
    ]{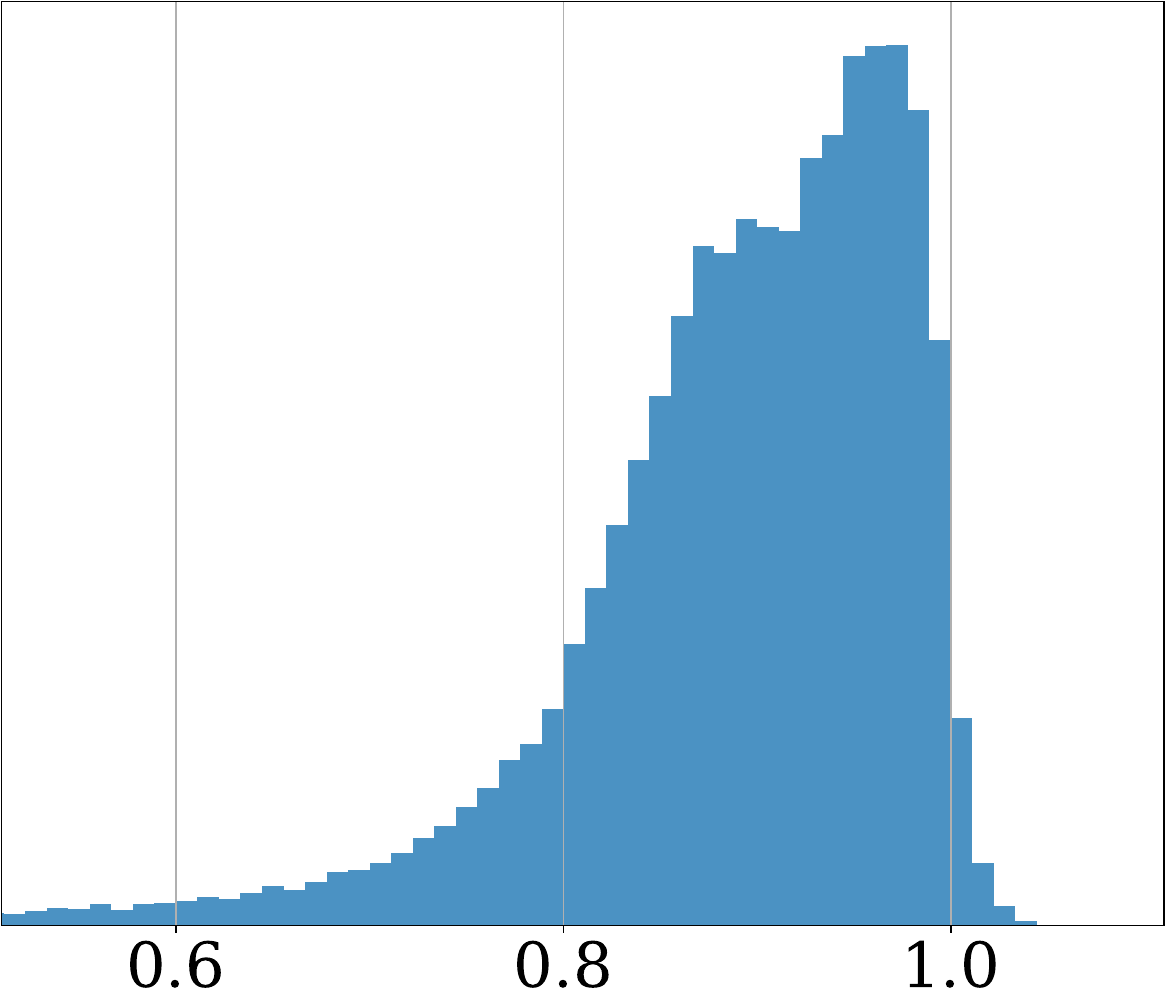}
        \put(50, -5){\makebox(0,0){$\|\grad f(\theta, x_i)\|$}}
    \end{overpic}
    \hspace{0.3cm}
    \begin{overpic}[
        height=\figheight,
    ]{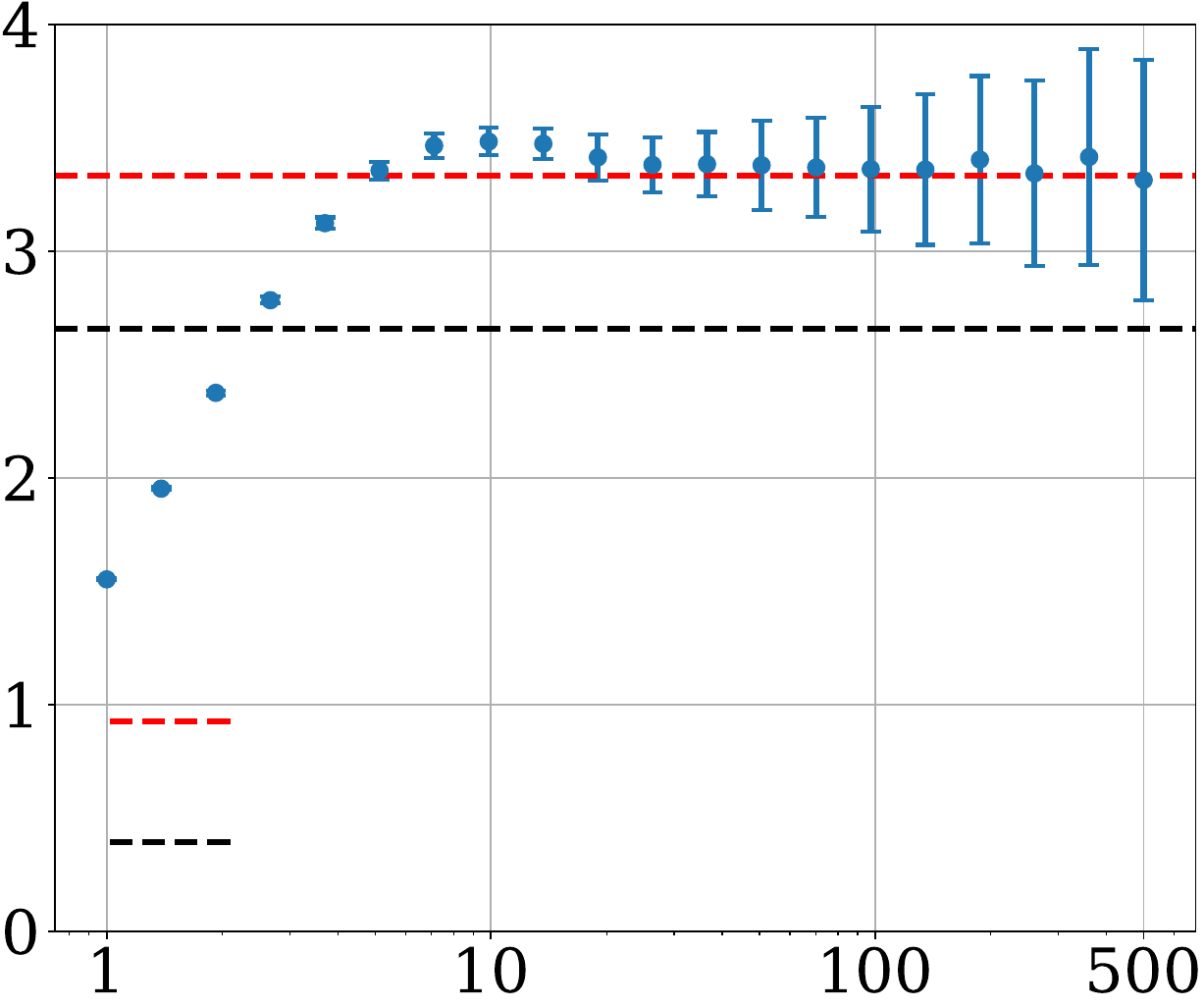}
        \put(21, 22){\scalebox{.8}{$\int_{\surface_\theta}\!\!\frac{1}{\|g\|}dS$}}
        \put(21, 12){\scalebox{.8}{$|\surface_\theta|$}}
        \put(-9, 30){\rotatebox{90}{$\SSAloss^\alpha$}}
    \end{overpic}

    \hspace{0.05cm}
    \begin{overpic}[
        height=\figheight, trim={.3cm 1cm 0 1cm}, clip
    ]{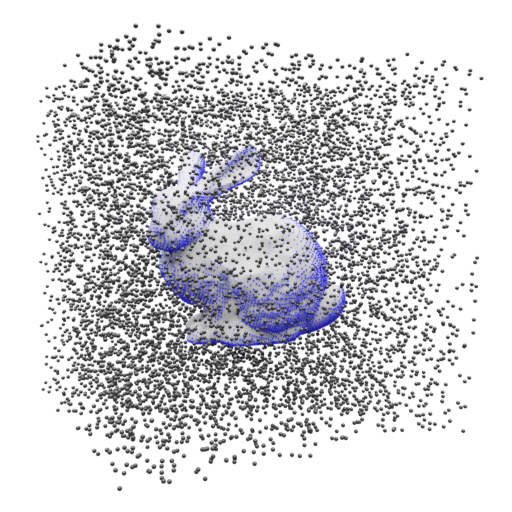}
        \put(-0, 50){\makebox(0,0){\rotatebox{90}{Final $S_\theta$}}}
        \put(50, -6){\makebox(0,0){$x_i$}}
    \end{overpic}    
    \hspace{-0.3cm}
    \begin{overpic}[
        height=\figheight,
    ]{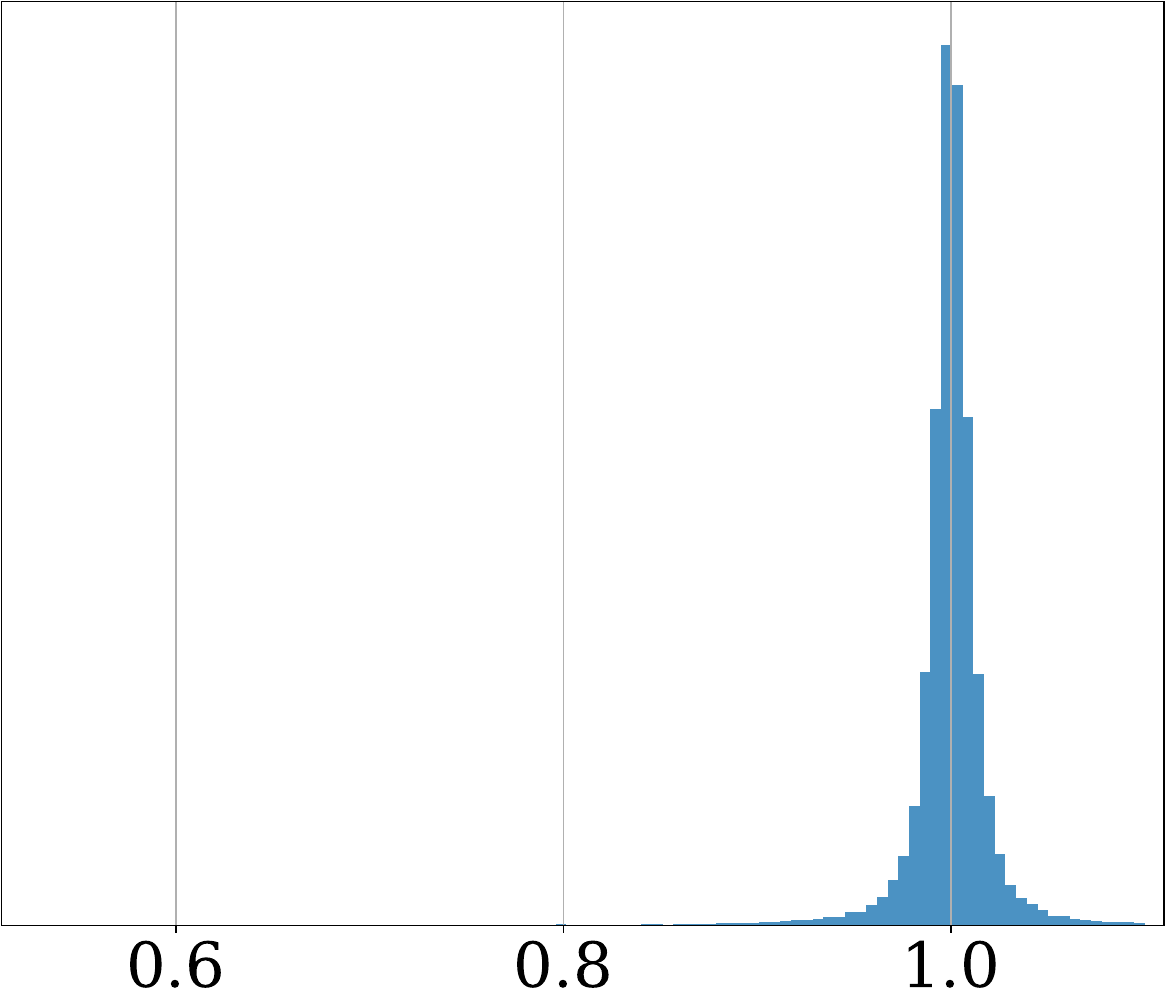}
        \put(50, -6){\makebox(0,0){$\|g(\theta, x_i)\|$}}
    \end{overpic}
    \hspace{0.3cm}
    \begin{overpic}[
        height=\figheight,
    ]{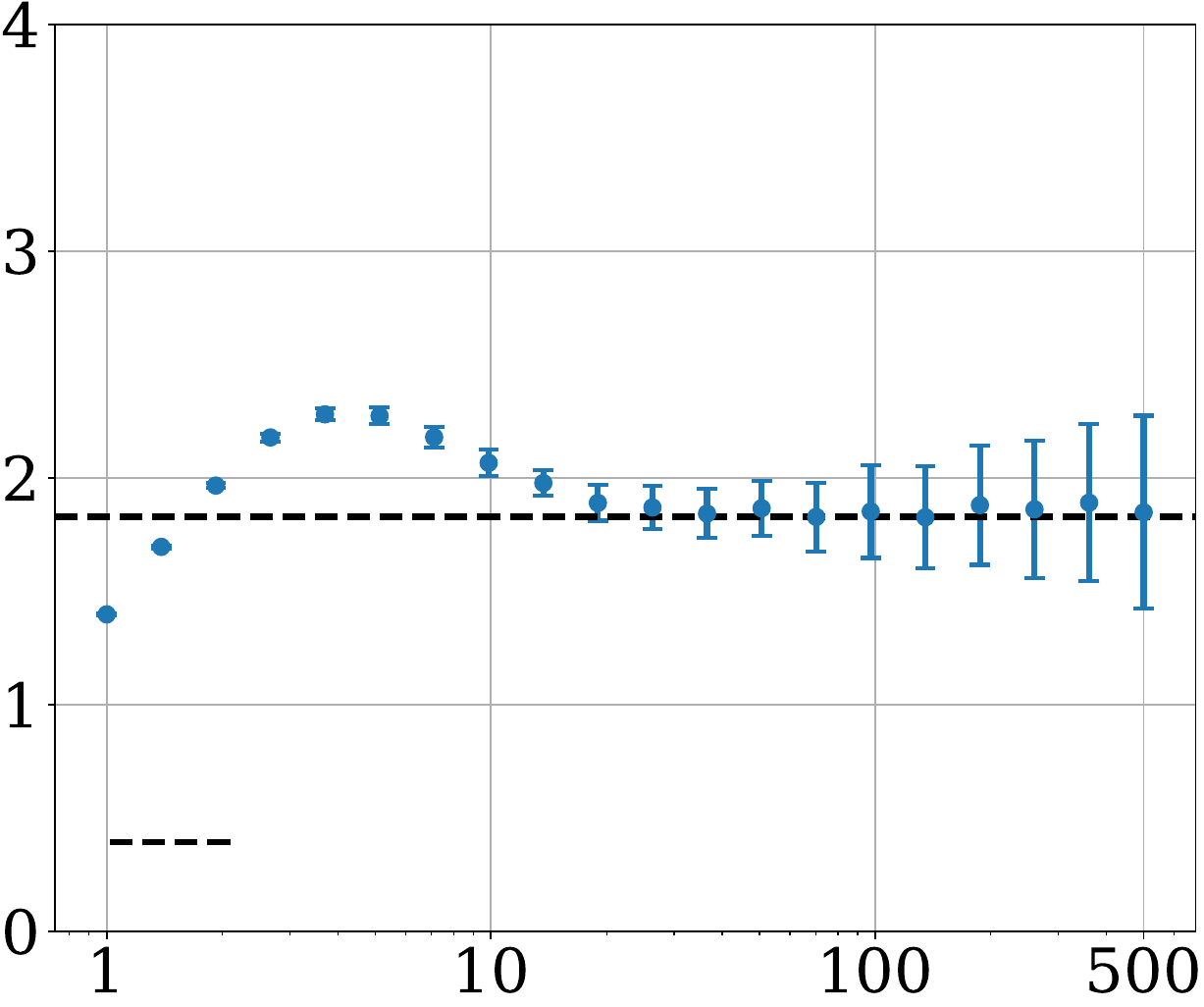}
        \put(55, -2){\makebox(0,0){$\alpha$}}
        \put(23, 12){\scalebox{.8}{$|\surface_\theta|$}}
        \put(-9, 30){\rotatebox{90}{$\SSAloss^\alpha$}}
    \end{overpic}
    \vspace{.1cm}
    \caption{Comparison between \cref{thm:ssa} and the sample approximation from \cref{eqn:ssa-sample}. \textbf{Top row:} Neural implicit surface at initialization. \textbf{Bottom row:} implicit surface after 40K iterations of DiffCD on the bunny point cloud marked in blue. In each case, we sample 100 sets of 5K points uniformly from $\Omega$ marked as black dots (left) and show the resulting gradient norm histograms (center) and the mean sample approximations of $\SSAloss^\alpha$, $\pm$ 1 standard deviation for a range of $\alpha$ (right). Dashed lines indicate the actual values of the surface integral and surface area, as computed from a mesh obtained using marching cubes \cite{lewiner2003efficientMarchingCubes}.}%
    \label{fig:ssa-sample}%
\end{figure}
In this section, we investigate the sample approximation from \cref{eqn:ssa-sample}:
\eq{
    \sampleSSAloss^\alpha(\theta)=\frac{|\Omega|}{K}\frac{\alpha}{2}\sum_{i=1}^Ke^{-\alpha|f(\theta, x_i)|}.
}
We have shown in \cref{thm:ssa} that $\sampleSSAloss^\alpha$ approaches a surface integral in the idealized setting of $\alpha=K=\infty$.
Yet, it is natural to ask whether the approximation is valid in practice.
To this end, we perform a numerical experiment, as summarized in \cref{fig:ssa-sample}.
Specifically, we sample multiple point sets (each with 5K points), from which we compute $\sampleSSAloss^\alpha$ for various values of $\alpha$.

The first row of \cref{fig:ssa-sample} shows the sample approximations for a surface before training.
With weights obtained using the initialization scheme from \citet{gropp2020implicit}.
While the initialization ensures that the initial surface is roughly a sphere, the neural field $f(\theta, \argdot)$ is far from an SDF.
In fact, the histogram reveals that most gradient norms are strictly smaller than 1.
In this case, the sample average indeed converges to the surface integral weighted by the inverse gradient; $\int_{\surface_\theta}1/\|g\|\diff S \neq |\surface_\theta|$, as predicted by \cref{thm:ssa}.

The second row of \cref{fig:ssa-sample} shows the same experiment after fitting $f(\theta, \argdot)$ to a point cloud using DiffCD.
In this case, most gradient norms, $\|g\|$, are close to 1 due to the eikonal loss.
As a result, the SDF approximation is valid, and the sample average approaches the true surface area $|\surface_\theta|$.
This confirms that minimizing the sample approximation $\sampleSSAloss$ along with the eikonal loss is equivalent to minimizing the surface area of the implicit surface.

In both cases, we can also note that increasing $\alpha$ leads to a larger sample variance.
This is because any given sample is less likely to land close enough to the surface to contribute a nonzero value to the average. 
On the other hand, when $\alpha$ is too small, the sample average is more or less independent of the surface.
We use $\alpha=100$ for our experiments, as proposed by \citet{sitzmann2019siren}.
We can verify that this value is reasonable by noting that it allows the sample average to converge without having an overly large variance.

\section{A toy example}
\label{sec:toy-example}
We can demonstrate how the SIREN \cite{sitzmann2019siren} loss can lead to disappearing surfaces using the following toy example of a 2D circle. 
Suppose the points $\tilde x_i$ all lie on a circle of unknown radius $r$ in $\Omega=\mathbb{R}^2$.
We then parameterize the neural SDF as $f(\theta, x) = \|x\| - \theta$.
We can confirm that, as $\grad f(\theta, x) = \frac{x}{\|x\|}$, the field $f(\theta, \argdot)$ satisfies the eikonal constraint everywhere except at $x=0$.
Furthermore, the 0-level set is a circle of radius $\theta$, so clearly the true shape is obtained by setting $\theta = r$.

Plugging $f(\theta, \argdot)$ into \cref{eqn:siren-loss}, we obtain a closed-form expression:
\eq{
    \IGRSSAloss^\alpha(\theta) &= \frac{1}{n}\sum_{i=1}^n|\|\tilde x_i\| - \theta|^p + \mu\frac{\alpha}{2}\int_{\mathbb{R}^2}e^{-\alpha|f(\theta, x)|}\diff x
\\  &= |r - \theta|^p + \pi\mu\alpha\Bigg(\int_0^\theta re^{-\alpha(\theta - r)}\diff r + \int_\theta^{\infty} re^{-\alpha(r - \theta)}\diff r\Bigg)
\\  &= |r - \theta|^p + \mu(2\pi\theta + \frac{\pi}{\alpha}e^{-\alpha\theta}).
}
For a large $\alpha$, the loss function tends towards $\SSAloss^\infty(\theta) = |r - \theta|^p + \mu 2\pi\theta$, which is minimized by
\eq{
    &\theta^* = \begin{cases}
        r & \mu < \frac{1}{2\pi}
    \\  0 & \mu > \frac{1}{2\pi}
    \end{cases}, \quad&\text{for $p=1$;}
\\  &\theta^* = \max\{0, r - \pi\mu\}, \;\quad&\text{for $p=2$}.
}
With the mean squared error ($p=2$), the estimated radius $\theta^*$ is always strictly less than the actual radius $r$ and shrinks linearly with the weighing parameter $\mu$ until the surface collapses to a single point at $x = 0$ for $\mu\geq r/\pi$.
In the mean error case ($p=1$), the true radius will be recovered perfectly if $\mu$ is below the threshold $1/2\pi$.
However, the radius will suddenly collapse to a point when $\mu$ goes above the threshold. Although the threshold can be determined in closed form in this example, it will be generally unpredictable and depend on the magnitude of the gradient of the surface area \wrt the surface parameters $\theta$.

\section{Parameter ablation and eikonal weight smoothing}
\label{sec:eikonal_weight}

\begin{figure}[t]
    \vspace{.4cm}
    \hspace{.005\columnwidth}
    \begin{overpic}[width=.99\columnwidth]{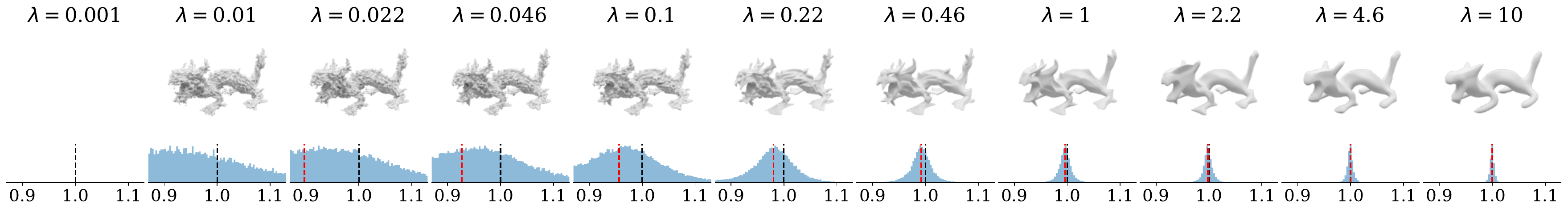}
        \put(50, 15){\rotatebox{0}{\makebox(0,0){\scriptsize{\textbf{Medium Noise}}}}}
        \put(-0.2, 8){\rotatebox{90}{\makebox(0,0){\scalebox{0.5}{Shape}}}}
        \put(-0.2, 3.5){\rotatebox{90}{\makebox(0,0){\scalebox{0.5}{$\|g\|$}}}}
    \end{overpic}

    \vspace{.6cm}
    \hspace{.005\columnwidth}
    \begin{overpic}[width=.99\columnwidth, trim={0 0 0 1cm}, clip]{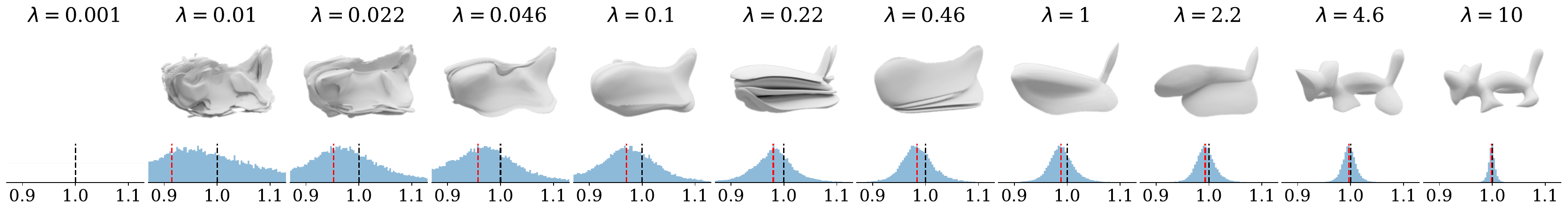}
        \put(50, 14){\rotatebox{0}{\makebox(0,0){\scriptsize{\textbf{Max Noise --} LR$=\!10^{-3}$}}}}
        \put(-0.2, 8){\rotatebox{90}{\makebox(0,0){\scalebox{0.5}{Shape}}}}
        \put(-0.2, 3.5){\rotatebox{90}{\makebox(0,0){\scalebox{0.5}{$\|g\|$}}}}
    \end{overpic}

    \vspace{.6cm}
    \hspace{.005\columnwidth}
    \begin{overpic}[width=.99\columnwidth, trim={0 0 0 1cm}, clip]{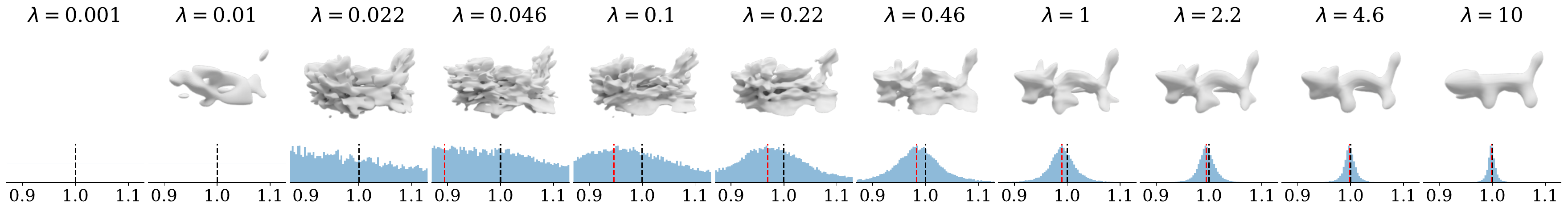}
        \put(50, 14){\rotatebox{0}{\makebox(0,0){\scriptsize{\textbf{Max Noise --} LR$=\!5\!\!\times\!\!10^{-5}$}}}}
        \put(-0.2, 8){\rotatebox{90}{\makebox(0,0){\scalebox{0.5}{Shape}}}}
        \put(-0.2, 3.5){\rotatebox{90}{\makebox(0,0){\scalebox{0.5}{$\|g\|$}}}}
    \end{overpic}
    \vspace{-0.4cm}
    \caption{Reconstructed shapes and gradient norm histograms obtained using DiffCD for various eikonal weights $\lambda$. A larger $\lambda$ leads to smoother surfaces, which prevents overfitting.}
    \label{fig:eikonal-original-xyzrgb_dragon_clean}
\end{figure}

\label{sec:eikonal-smoothing}
\begin{table}[t]
\scriptsize
\sisetup{detect-weight=true,detect-inline-weight=math}
\setlength\tabcolsep{4pt}
\begin{tabularx}{.6\linewidth}{
X
S[table-format=1.1, table-column-width=3.1em]@{\hspace{0.1em}}
S[table-format=1.1, table-column-width=3.1em]@{\hspace{0.1em}}
S[table-format=1.3, table-column-width=3.1em]@{\hspace{0.1em}}
S[table-format=1.3, table-column-width=3.1em]@{\hspace{0.1em}}
S[table-format=2.1,      table-column-width=3.1em]
}
\toprule
        &&
        & \multicolumn{3}{c}{Medium Noise}\\
        \cmidrule(lr){4-6}
  Methods
  & $\lambda$ & LR 
  & {CD} & {CD$^2$}  & {CA ($^\circ$)} 
  \\
\midrule
Neural-Pull & {--} & ${10^{-3}}$
& 0.924 & 4.086 & 41.7 
\\

IGR & 0.1 & ${10^{-3}}$
& 1.454 & 13.185 & 49.8 
\\

SIREN $(\mu\equal 0.033)$ & 0.1 & ${10^{-3}}$
& 0.782 & 1.348 & 45.4 
\\

SIREN $(\mu\equal 0.33)$ & 0.1 & ${10^{-3}}$
& 0.719 & 1.639 & 33.8 
\\

DiffCD (ours) & 0.1 & ${10^{-3}}$
& \textbf{0.654} & \textbf{0.698} & 35.7 
\\
\midrule

IGR & 0.5 & ${10^{-3}}$
& 1.504 & 15.914 & 35.1 
\\

SIREN $(\mu\equal 0.033)$ & 0.5 & ${10^{-3}}$
& 0.963 & 3.647 & 32.5 
\\

SIREN $(\mu\equal 0.33)$ & 0.5 & ${10^{-3}}$
& 0.739 & 1.724 & 26.8 
\\

DiffCD (ours) & 0.5 & ${10^{-3}}$
& 0.690 & 1.069 & \textbf{26.6} 
\\

\bottomrule
\end{tabularx}
\hspace{.3cm}
\newcommand{\figwidth}{.15\columnwidth}
\begin{tabular}{cc}
    \vspace{1cm}&\\
    \begin{overpic}[width=\figwidth]{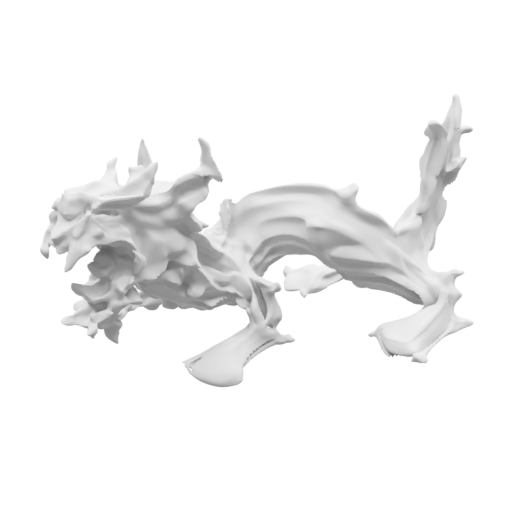}
        \put(50, 90){\makebox(0,0){
            \scriptsize{SIREN \cite{sitzmann2019siren}}
        }}
        \put(-15, 50){\makebox(0,0){\rotatebox{90}{
            \scriptsize{$\lambda\equal0.1$}
        }}}
    \end{overpic}& 
    \rotatebox{0}{\begin{overpic}[width=\figwidth]{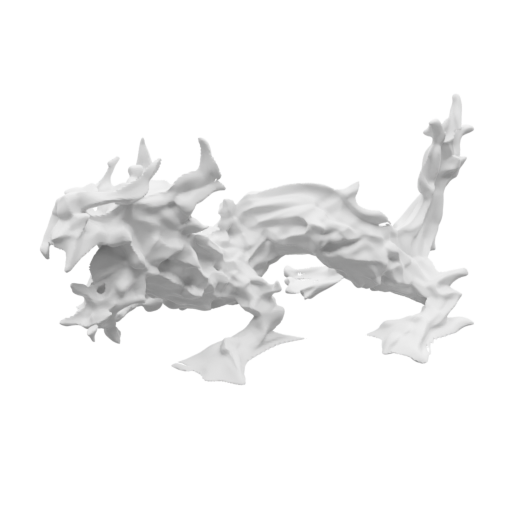}
        \put(50, 90){\makebox(0,0){
            \scriptsize{DiffCD (ours)}
        }}
    \end{overpic}}
    \vspace{-0.3cm}\\
    \begin{overpic}[width=\figwidth]{figures/original_siren/xyzrgb_dragon_clean.png}
        \put(-15, 50){\makebox(0,0){\rotatebox{90}{
            \scriptsize{$\lambda\equal0.5$}
        }}}
    \end{overpic}&
    \begin{overpic}[width=\figwidth]{figures/original_diffcd/xyzrgb_dragon_clean.png}
    \end{overpic}
\end{tabular}  

\caption{\textbf{Parameter ablation on FAMOUS in the medium noise setting.} Lower is better. The best performing optimization-based method in each category is marked in bold.} 
\label{table:ablation-medium}
\vspace{-0.5cm}
\end{table}

\begin{table}[t]
\scriptsize
\sisetup{detect-weight=true,detect-inline-weight=math}
\setlength\tabcolsep{4pt}

\sisetup{detect-weight,mode=text}
\renewrobustcmd{\bfseries}{\fontseries{b}\selectfont}
\renewrobustcmd{\boldmath}{}
\newrobustcmd{\B}{\bfseries}

\begin{tabularx}{.6\linewidth}{
        X
        S[table-format=1.1, table-column-width=3.1em]@{\hspace{0.1em}}
        S[table-format=1.1, table-column-width=3.1em]@{\hspace{0.1em}}
        S[table-format=1.3, table-column-width=3.1em]@{\hspace{0.1em}}
        S[table-format=3.3, table-column-width=3.1em]@{\hspace{0.1em}}
        S[table-format=2.1,      table-column-width=3.1em]
    }
    \toprule
            &&
            & \multicolumn{3}{c}{Max Noise}\\
            \cmidrule(lr){4-6}
      Methods
      & $\lambda$ & LR 
      & {CD} & {CD$^2$}  & {CA ($^\circ$)} 
      \\
    \midrule
    Neural-Pull & {--} & ${10^{-3}}$
    & 1.580 & 5.101 & 66.9 
    \\
    
    IGR & 0.1 & ${10^{-3}}$
    & 8.613 & 281.441 & 87.1 
    \\
    
    SIREN $(\mu\equal 0.033)$  & 0.1 & ${10^{-3}}$
    & 2.828 & 16.140 & 85.9 
    \\
    
    SIREN $(\mu\equal 0.33)$  & 0.1 & ${10^{-3}}$
    & 1.845 & 5.936 & 79.2 
    \\
    
    DiffCD (ours)  & 0.1 & ${10^{-3}}$
    & 2.042 & 7.888 & 84.3 
    \\
    \midrule
    
    Neural-Pull & {--} & ${5\!\!\times\!\!10^{-5}}$
    & 1.926 & 11.114 & 44.9 
    \\
    
    IGR  & 1.0 & ${5\!\!\times\!\!10^{-5}}$
    & 2.553 & 27.912 & 45.7 
    \\
    
    SIREN $(\mu\equal 0.033)$ & 1.0 & ${5\!\!\times\!\!10^{-5}}$
    & 2.079 & 14.268 & 44.9 
    \\
    
    SIREN $(\mu\equal 0.33)$ & 1.0 & ${5\!\!\times\!\!10^{-5}}$
    & 1.658 & 7.698 & 41.1 
    \\
    
    DiffCD (ours) & 1.0 & ${5\!\!\times\!\!10^{-5}}$
    & \textbf{1.416} & \B 4.306 & \textbf{37.4} 
    \\
    
    \bottomrule
\end{tabularx}
\hspace{.3cm}
\newcommand{\figwidth}{.15\columnwidth}
\begin{tabular}{cc}
    \vspace{1cm}&\\
    \begin{overpic}[width=\figwidth]{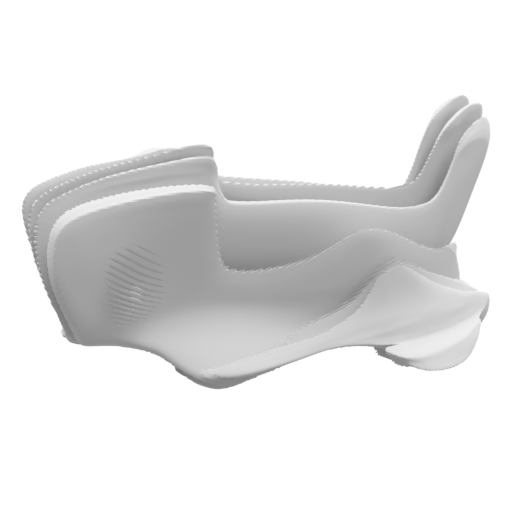}
        \put(50, 90){\makebox(0,0){
            \scriptsize{SIREN \cite{sitzmann2019siren}}
        }}
        \put(-35, 15){\rotatebox{90}{
            \Centerstack[l]{\scriptsize{$\lambda\equal0.1$}\\LR$=10^{\text{-}3}$}
        }}
    \end{overpic}& 
    \rotatebox{0}{\begin{overpic}[width=\figwidth]{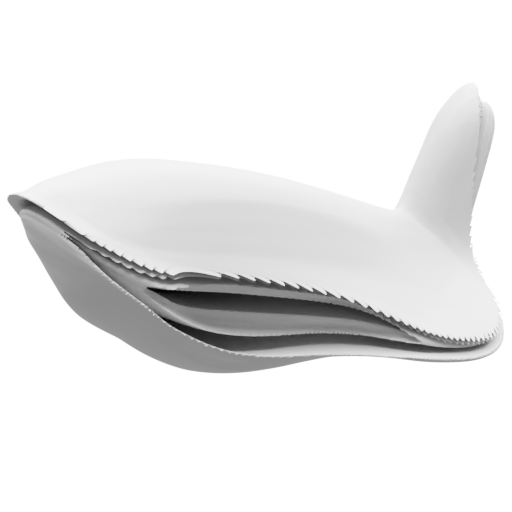}
        \put(50, 90){\makebox(0,0){
            \scriptsize{DiffCD (ours)}
        }}
    \end{overpic}}
    \vspace{-0.1cm}\\
    \begin{overpic}[width=\figwidth]{figures/extra_noisy_siren/xyzrgb_dragon_clean.png}
        \put(-35, 15){\rotatebox{90}{
             \Centerstack[l]{\scriptsize{$\lambda\equal1$}\\LR$=5\!\!\times\!\!10^{\text{-}5}$}
        }}
    \end{overpic}&
    \begin{overpic}[width=\figwidth]{figures/extra_noisy_diffcd/xyzrgb_dragon_clean.png}
    \end{overpic}
\end{tabular}  

\caption{\textbf{Parameter ablation on FAMOUS in the max noise setting.} Lower is better. The best performing optimization-based method in each category is marked in bold.} 
\label{tab:ablation-max}
\vspace{-0.5cm}
\end{table}
We find that the value of the eikonal weight, $\lambda$, has a significant impact on the overall smoothness of reconstructed shapes.
In particular, larger $\lambda$ tends to produce smoother shapes.
This is likely due to the fact that non-flat surfaces have more non-differentiable points in the distance field, where the gradient changes direction abruptly.
Such points are difficult to approximate without violating the eikonal constraint, and they are therefore naturally avoided by the optimization process if $\lambda$ is sufficiently large.
We demonstrate this effect qualitatively in \cref{fig:eikonal-original-xyzrgb_dragon_clean} the noisy scenarios.
When $\lambda$ is too small, the field collapses to the degenerate solution $f(\theta, \argdot) = 0$.
As $\lambda$ increases, a distinct 0-level set eventually appears.
However, the reconstructed shapes are initially ``wrinkly'', since nothing prevents the surface from overfitting to the noise.
With larger $\lambda$, the surfaces become smoother, and the gradient norms concentrate around the value 1.
On the other hand, setting $\lambda$ too large results in over-smoothing, so the optimal value for any given shape will lie somewhere in between.
We can also note that, in contrast to $\IGRSSAloss$, increasing $\lambda$ does not lead to the surface disappearing. 

We find that using the base settings of $\lambda\equal 0.1$ and a learning rate (LR) of $10^{-3}$ works well for the noise-free setting.
However, in the ``medium noise'' and ``max noise'' settings, we observe that these parameter settings result in an under-regularized problem, and the surface tends to overfit the noise.

\inparagraph{Medium noise.} In the medium noise setting, we observe that optimization-based methods tend to overfit to the point cloud by introducing ``wrinkles'' in the surface such that it passes through all the points.
We demonstrate this phenomenon quantitatively and qualitatively in \cref{table:ablation-medium}.
The overfitting is particularly noticeable in the Chamfer angle (CA) metric, where all methods have a CA of above 35$^\circ$ when using $\lambda\equal 0.1$.
In order to reduce the degree of overfitting, we increase the eikonal weight to $\lambda\equal0.5$.
The larger eikonal weight results in smoother surfaces across all methods, with a substantial decrease in the CA metric, although at the cost of a slight increase in CD and CD$^2$.

\inparagraph{Max noise.} In the maximum noise setting, we observe a different kind of overfitting.
Namely, the surface tends to ``fold'' into several layers.
This folding does in fact result in a lower Chamfer distance (both symmetric and one-sided) between the implicit surface and the point cloud.
Therefore, mitigating this phenomenon again requires regularization.
In this case, we set $\lambda = 1$ and LR$=5\!\!\times\!\!10^{-5}$.
We verify in \cref{tab:ablation-max} that these settings yield better results across all methods.

\section{Comparison to Neural-Pull}
\label{sec:neural-pull}
\begin{figure}[t]
    \vspace{1cm}
    \newcommand{\figwidth}{.24\columnwidth}
    \centering
    \begin{tabular}{cccc}
        \begin{overpic}[width=\figwidth]{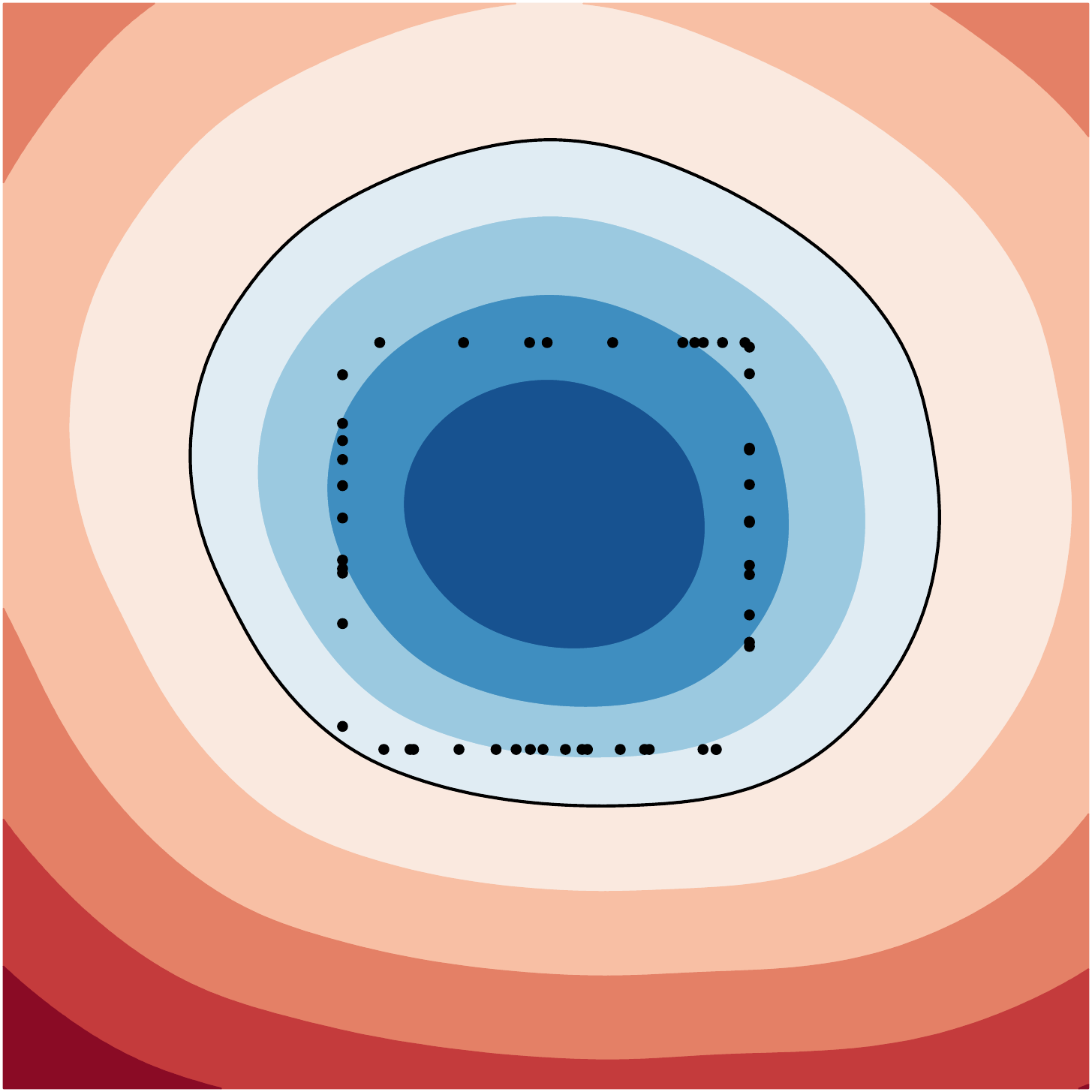}
            \put(50, 105){\makebox(0,0){\rotatebox{0}{
                \scriptsize{\textbf{Initialization}}
            }}}
        \end{overpic}&
        \begin{overpic}[width=\figwidth]{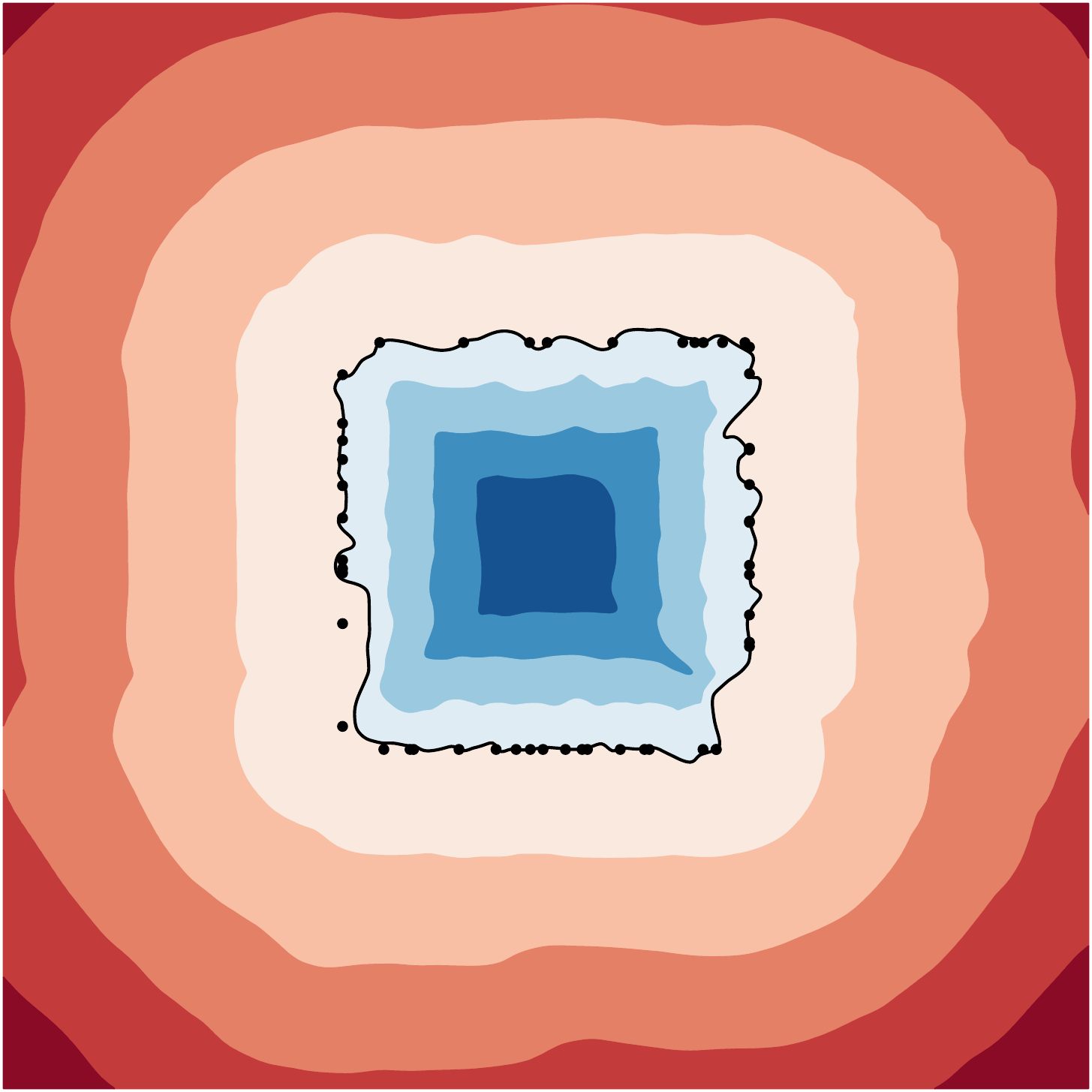}
            \put(50, 105){\makebox(0,0){\rotatebox{0}{
                \scriptsize{\textbf{\neupull \cite{Ma2020NeuralPull}}}
            }}}
        \end{overpic}&
            \begin{overpic}[width=\figwidth]{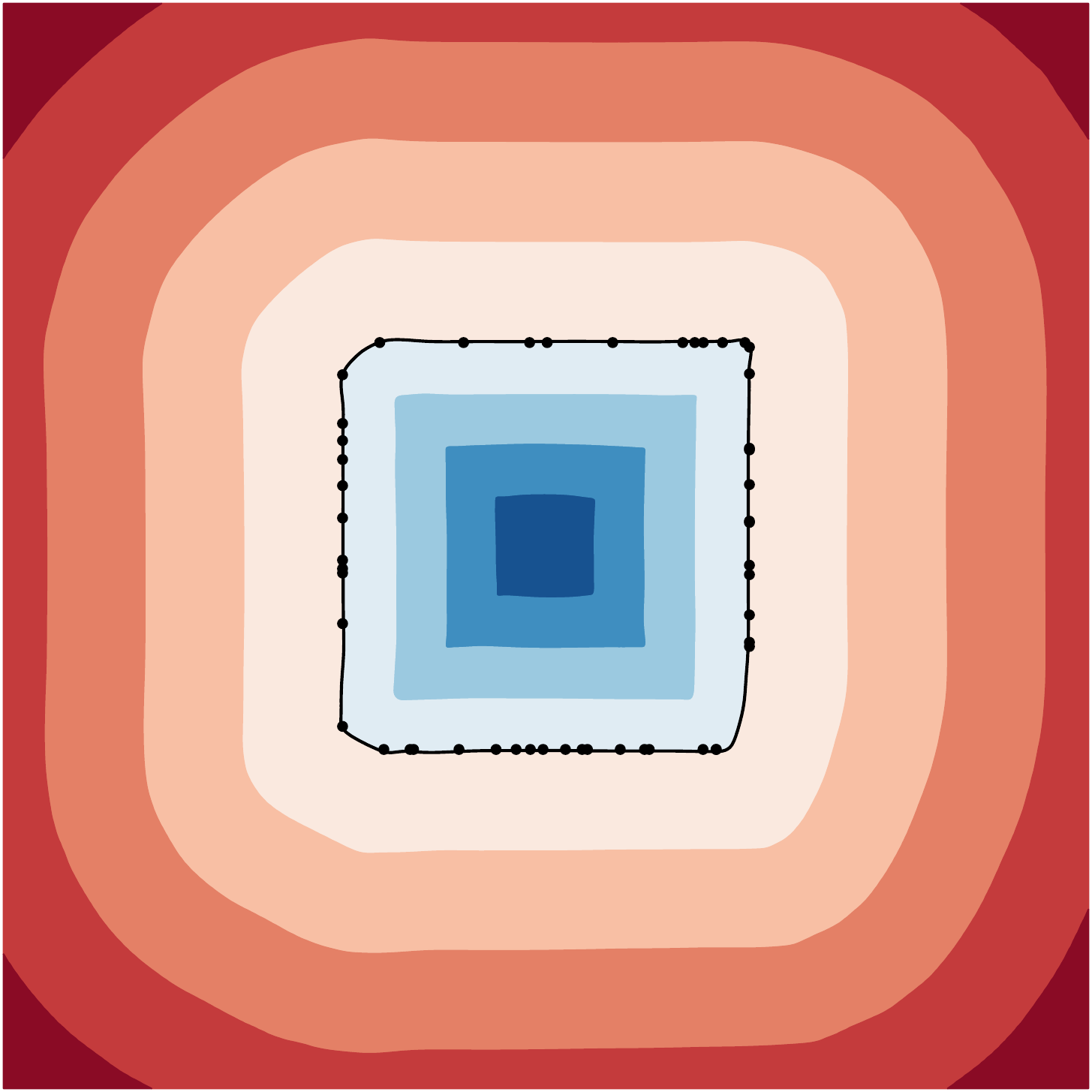}
            \put(50, 105){\makebox(0,0){\rotatebox{0}{
                \scriptsize{\textbf{IGR \cite{gropp2020implicit}}}
            }}}
        \end{overpic}&
        \begin{overpic}[width=\figwidth]{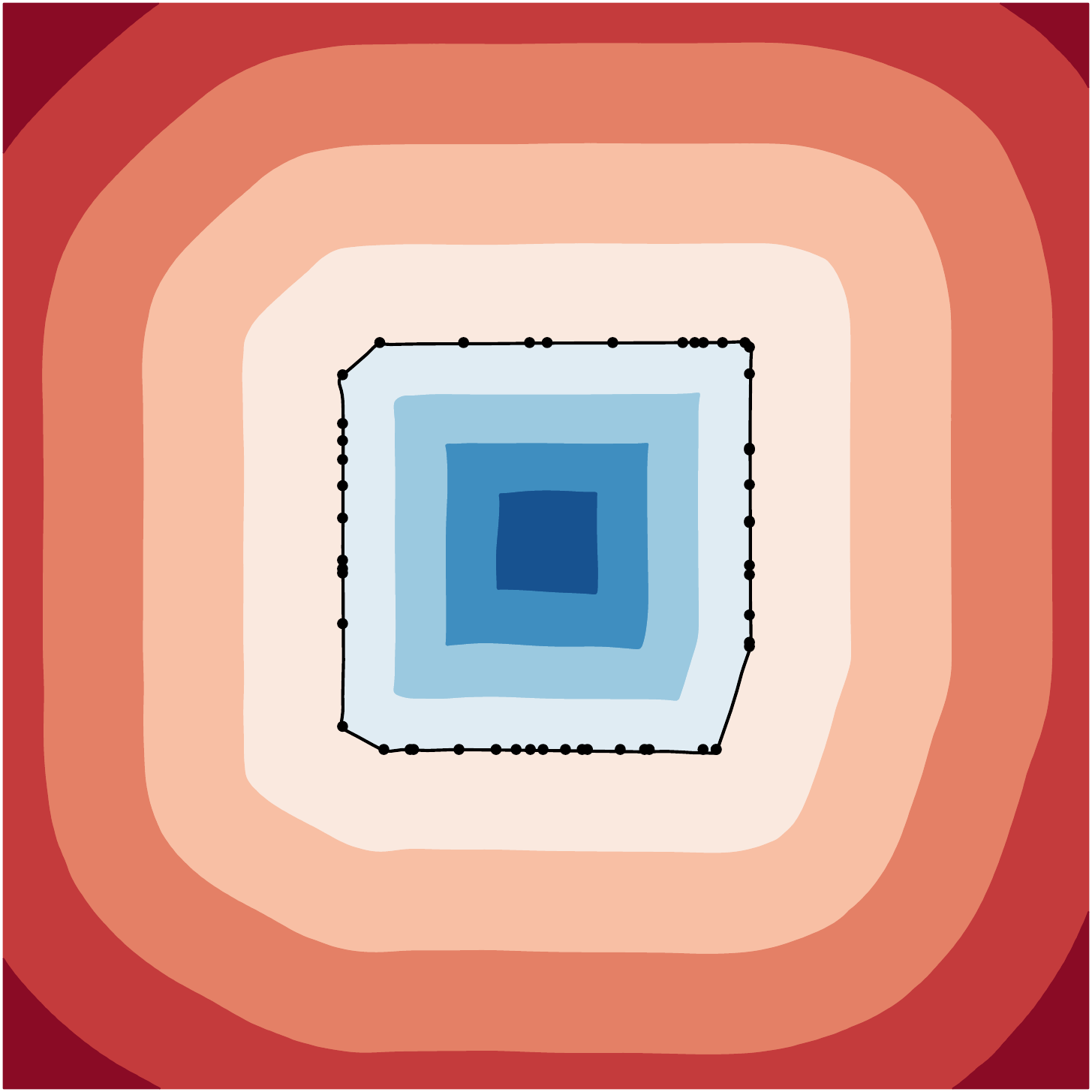}
            \put(50, 105){\makebox(0,0){\rotatebox{0}{
                \scriptsize{\textbf{DiffCD (ours)}}
            }}}
        \end{overpic}
    \end{tabular}
    \vspace{-0.4cm}
    \caption{Comparison between Neural-Pull and other optimization-based methods for a sparse 2D point cloud. \neupull \cite{Ma2020NeuralPull} tends to pull the surface apart, causing wrinkles between training points. In contrast, methods based on the implicit points-to-surface distance (\cf \cref{eqn:igr-loss}) tend to smoothly interpolate the points.}
    \label{fig:sparse-box}
\end{figure}

\neupull \cite{Ma2020NeuralPull} tackles the surface reconstruction problem by defining a pulling operation:
\eq{
    \label{eqn:pull-operation}
    \pull{\theta}{x} = x - f(\theta, x)\frac{\grad f(\theta, x)}{\|\grad f(\theta, x)\|}
}
which ``pulls'' the point $x$ a distance of $|f(\theta, x)|$ along the direction of the gradient.
The pulled point will be close to the implicit surface $S_\theta$ as long as $f(\theta, x)$ approximates an SDF.
The idea behind \neupull is then to sample a set of points, $\globalsample\in\globalsamples$, and minimize the average distance\footnote{While Eq. 2 in \citet{Ma2020NeuralPull} denotes the mean \textit{squared} distance between the point correspondences, the released implementation uses the mean distance. We opt for the mean distance to match the implementation.} between the pulled location of each sample point and the corresponding closest point in the point cloud $\pointcloud$:
\eq{
\label{eqn:neural-pull-loss}
    \NPloss^p(\theta) = 
        \frac{1}{|\globalsamples|}\sum_{\globalsample\in \globalsamples}\|\closest{\pointcloud}{\globalsample} - \pull{\theta}{\globalsample}\|.
}
Here, $\closest{A}{\globalsample}$ denotes the closest point to $\globalsample$ in the set $A$:
\eq{
    \closest{A}{\globalsample} = \min_{x\in A}\|x - \globalsample\|.
}
With this definition, we have $\pull{\theta}{\globalsample}\equal\closest{\surface_\theta}{\globalsample}$ in the case when $f(\theta, \argdot)$ is an SDF.
We can therefore view minimizing \cref{eqn:neural-pull-loss} as a method for sampling correspondences between the point cloud and the implicit surface, and then minimizing the resulting distances.\footnote{\cref{eqn:neural-pull-loss} cannot be directly interpreted as a Chamfer distance, as $\closest{\surface_\theta}{\globalsample}$ and $\closest{\pointcloud}{\globalsample}$ are not necessarily the closest points \wrt each other.}
This approach works well in cases where $\closest{\pointcloud}{x}$ is a good approximation of the closest point on the true surface $\surface$, as is the case when $\pointcloud$ densely samples $\surface$.
However, it introduces problems when the point cloud is noisy or sparse -- the scenario we addressed in this work.

We illustrate the drawbacks of \neupull for a sparse 2D point cloud in \cref{fig:sparse-box}.
In the sparse scenario, the sample points $\globalsample$ are more likely to end up \textit{in-between} the training points, resulting in the surface being pulled apart and introducing ``wrinkles''.
IGR \cite{gropp2020implicit} and DiffCD, which are based on the implicit points-to-surface loss (\cf \cref{eqn:igr-loss}), tend to instead smoothly interpolate between surface samples.

We note that our surface-to-points loss used in DiffCD also ``pulls'' the surface towards the point cloud.
However, the choice of the minimum norm gradient in \cref{sec:diffcd} ensures that the pulling stops when the direction towards the closest point in the point cloud is orthogonal to the surface normal. Ensuring that a smooth surface is obtained even for sparse point clouds.



\end{document}